%% file: neurips_2026.tex
\documentclass{article}

\usepackage[preprint]{neurips_2026}

\usepackage[utf8]{inputenc} 
\usepackage[T1]{fontenc}    
\usepackage{hyperref}       
\usepackage{url}            
\usepackage{booktabs}       
\usepackage{amsfonts}       
\usepackage{nicefrac}       
\usepackage{microtype}      
\usepackage{xcolor}         
\usepackage{algorithmic}
\usepackage{multirow}

\usepackage[ruled,vlined]{algorithm2e}
\usepackage{amsmath}
\usepackage{amssymb}

\usepackage{mathtools}
\usepackage{amsthm}
\usepackage{bbm}
\usepackage{tikz}
\usetikzlibrary{arrows.meta,positioning,decorations.pathmorphing,calc}

\usepackage[capitalize,noabbrev]{cleveref}

\theoremstyle{plain}
\newtheorem{theorem}{Theorem}[section]
\newtheorem{proposition}[theorem]{Proposition}
\newtheorem{lemma}[theorem]{Lemma}

\theoremstyle{definition}

\newtheorem{assumption}[theorem]{Assumption}
\theoremstyle{remark}

\title{Entropic Riemannian Neural Optimal Transport}

\author{%
  Alessandro Micheli \\
  Imperial College London\\
  London, UK \\
  \texttt{a.micheli19@imperial.ac.uk} \\
  \And
  Silvia Sapora \\
  University of Oxford \\
  Oxford, UK \\
  \texttt{silvia.sapora@stats.ox.ac.uk} \\
   \And
  Anthea Monod \\
  Imperial College London\\
  London, UK \\
  \texttt{a.monod@imperial.ac.uk} \\
    \And
  Samir Bhatt \\
  Imperial College London \\
  London, UK \\
  Statens Serum Institut \\
  Copenhagen, Denmark\\ 
  University of Copenhagen \\
  Copenhagen, Denmark\\
  \texttt{s.bhatt@imperial.ac.uk} \\
}

\begin{document}

\maketitle

\begin{abstract}
\input{sections/abstract}
\end{abstract}

\section{Introduction}
\input{sections/introduction}

\section{Background}
\label{sec:background}
\input{sections/background}

\section{Entropic Riemannian Neural Optimal Transport}
\label{sec:method}
\input{sections/method}

\section{Theoretical Guarantees}
\label{sec-theory}
\input{sections/theory}
\section{Experiments}
\input{sections/experiments}

\section{Discussion and Limitations}
\label{sec:discussion}
\input{sections/discussion}

 \section*{Acknowledgements}
 S.B. acknowledges support from the Novo Nordisk Foundation via The Novo Nordisk Young Investigator Award (NNF20OC0059309). S.B. acknowledges support from The Eric and Wendy Schmidt Fund For Strategic Innovation via the Schmidt Polymath Award (G-22-63345) which also supports A.Micheli. S.B. acknowledges support from the Novo Nordisk Foundation via the Global Pathogen Analysis Platform (GPAP) (NNF26SA0109818) which also supports A.Micheli. A.Monod is supported by the EPSRC AI Hub on Mathematical Foundations of Intelligence: An ``Erlangen Programme'' for AI No.~EP/Y028872/1. S. S. is supported by the Engineering and Physical Sciences Research Council EP/W524311/.

 \section*{Contribution Statements}
 Author contributions are reported using the CRediT (Contributor Roles Taxonomy).
  \begin{itemize}
     \item \textbf{Alessandro Micheli}: Conceptualization; Methodology; Software; Formal analysis; Supervision; Investigation; Project administration; Visualization; Validation; Writing – original draft; Writing – review \& editing.
     \item \textbf{Silvia Sapora}:  
     Writing – review \& editing.
     \item \textbf{Anthea Monod}:
     Writing – original draft; Writing – review \& editing.
     \item \textbf{Samir Bhatt}:
      Software; Funding acquisition; Resources; Visualization; Validation; Writing – original draft; Writing – review \& editing.
 \end{itemize}

\bibliographystyle{plain}
\bibliography{ref_update} 

\clearpage
\newpage
\appendix
\input{sections/appendix}


\end{document}

%% file: sections/abstract.tex
Many machine learning problems involve data supported on curved spaces such as
spheres, rotation groups, hyperbolic spaces, and general Riemannian manifolds,
where Euclidean geometry can distort distances, averages, and the resulting
optimal transport (OT) problem. Existing manifold OT methods have pursued
amortized out-of-sample maps, while entropic regularization has made discrete OT
more scalable, but these advantages have remained largely disjoint. We propose Entropic Riemannian Neural Optimal Transport (Entropic RNOT), a unified framework that combines intrinsic entropic OT with amortized out-of-sample evaluation on Riemannian manifolds.
Our method learns a single target-side
Schr\"odinger potential through a neural pullback parameterization, recovers the
induced Gibbs coupling and uses the resulting conditional laws to construct intrinsic transport surrogates. These include barycentric projections on Cartan--Hadamard
manifolds and heat-smoothed conditional surrogates on complete stochastically
complete manifolds, the latter turning possibly atomic target laws into
absolutely continuous ones. For fixed regularization \(\varepsilon>0\), we prove
that the proposed hypothesis class recovers the entropic optimal coupling in
strong probabilistic metrics. As consequences, barycentric surrogates converge
in \(L^2\), while heat-smoothed surrogates are stable at fixed heat time and
asymptotically unbiased as the heat time vanishes. The guarantees hold for
compactly supported data on possibly noncompact manifolds. Empirically, our
method matches or improves over Euclidean, tangent-space, and log-Euclidean
baselines on benchmarks over \(\mathbb S^2\), \(\mathrm{SO}(3)\),
\(\mathrm{SPD}(3)\), \(\mathrm{SE}(3)\), and \(\mathbb H^2\), scales favorably
relative to discrete manifold Sinkhorn, and in a protein–ligand docking application, refines poses on $\mathrm{SE}(3)$ without retraining or per-instance optimization.

%% file: sections/introduction.tex
Optimal transport (OT) provides a principled way to compare and transform
probability distributions, and has become a powerful tool in modern machine
learning. In many applications, however, the data are not naturally Euclidean:
directions lie on spheres, orientations on rotation groups, rigid poses on
\(\mathrm{SE}(3)\), and covariance descriptors on SPD manifolds. In such
settings, Euclidean approximations can distort both distances and transport
summaries, motivating OT methods defined intrinsically on Riemannian manifolds.

Recent work has begun to extend neural OT to this setting. A central motivation
is \emph{amortization}: after training, a learned transport mechanism should
generalize to unseen data without requiring a new OT solve at test time. Recent
manifold OT models~\cite{micheli2026riemannianneuraloptimaltransport} pursue
intrinsic neural parameterizations of transport maps with out-of-sample
generalization, including constructions based on implicit cost-concave
potentials~\cite{rezende2021implicitriemannianconcavepotential}. However, these
approaches also highlight a key computational difficulty: intrinsic manifold OT
may still require iterative inner optimization, since essential quantities are
defined through optimization subproblems rather than closed-form evaluation.

A complementary route to tractability is \emph{entropic regularization}, which
replaces OT by a smoother problem that is much easier to solve, notably via
Sinkhorn iterations~\cite{cuturisinkhorn}. This greatly improves scalability
and numerical stability, but does not by itself provide an amortized model:
each new instance still typically requires solving a new regularized OT
problem. Amortization and entropic regularization therefore address distinct
challenges, namely reusability at test time and per-instance computational
cost.

We combine both ideas in an intrinsic geometric setting. Specifically, we
introduce \emph{Entropic Riemannian Neural Optimal Transport (Entropic RNOT)},
an intrinsic neural framework for the \emph{static} entropic OT problem on
possibly noncompact
Riemannian manifolds that aims to learn a reusable manifold-aware transport
model while retaining the smoothing and scalability benefits of entropic
regularization. To the best of our knowledge, ours is the first intrinsic
neural framework for this problem with amortized out-of-sample evaluation.

Our method is based on the semidual formulation of entropic OT: we learn a
target-side Schr\"odinger potential through a neural pullback parameterization
and recover the induced Gibbs coupling. We then study intrinsic transport
surrogates extracted from the conditional laws of this coupling. In the
Cartan--Hadamard setting, these conditionals define a deterministic transport
surrogate by barycentric projection. More generally, on complete
stochastically complete manifolds, heat smoothing provides a canonical
Riemannian way to turn possibly atomic conditional target laws, such as those
arising from finite samples or finite target supports, into absolutely
continuous conditional distributions.

For fixed \(\varepsilon>0\), we show that the proposed hypothesis class admits
approximating sequences whose induced Gibbs plans recover the entropic optimal
coupling. As stable consequences of this plan recovery, barycentric
projections recover the associated deterministic surrogate in \(L^2(\mu)\) on
Cartan--Hadamard manifolds, while heat-smoothed conditional laws are stable at
every fixed heat time and have vanishing population-level smoothing bias as the
heat time tends to zero.

Empirically, we evaluate our framework on synthetic and real manifold-valued
transport tasks. On benchmarks over the compact and noncompact manifolds
\(\mathbb S^2\), \(\mathrm{SO}(3)\), \(\mathrm{SPD}(3)\), \(\mathrm{SE}(3)\),
and \(\mathbb H^2\), our method is consistently competitive with or superior
to ambient Euclidean, tangent-space, and log-Euclidean baselines. Scalability
experiments show substantially better wall-clock, memory, and inference
scaling than discrete manifold Sinkhorn in the large-support regime. On a
real-world, rigid-pose refinement task on \(\mathrm{SE}(3)\) for
protein--ligand docking in CrossDocked2020, a single model trained on pooled
complexes refines held-out docking poses toward the near-native region without
crystal supervision during training or inference and without per-instance
optimization. This reduces top-1 RMSD from 11.24 \AA{} to 3.47 \AA{}, improves
success at 2 \AA{} from 10.3\% to 75.9\%, and substantially outperforms
physics-based minimization.

\emph{Our contributions are threefold.} First, we introduce Entropic RNOT, an intrinsic neural
framework for the \emph{static} entropic OT problem on Riemannian manifolds
that combines the semidual formulation with amortized out-of-sample evaluation.
Second, for fixed \(\varepsilon>0\), we prove recovery of the entropic optimal
coupling in strong probabilistic metrics, and derive stable recovery
results for intrinsic transport surrogates: barycentric projections converge in
\(L^2(\mu)\) in the Cartan--Hadamard setting, while heat-smoothed conditional
laws are stable at fixed heat time and asymptotically unbiased as the heat time
vanishes. Third, we demonstrate empirically strong transport quality across
diverse manifold geometries, more favorable scaling than discrete manifold
Sinkhorn, and effective real-world \(\mathrm{SE}(3)\) post-docking pose
refinement.

%% file: sections/background.tex
We briefly introduce the notation and constructions used throughout the paper;
a more detailed review of entropic optimal transport on Riemannian supports is
provided in Appendix~\ref{app:review-eot}.

Throughout, \((\mathcal M,g)\) denotes a complete \(p\)-dimensional
Riemannian manifold with geodesic distance \(d\), and
\(\mathcal P(\mathcal M)\) denotes the set of Borel probability measures on
\(\mathcal M\). For \(r\ge 1\), we write \(\mathcal P_r(\mathcal M)\) for the
space of probability measures on \(\mathcal M\) with finite \(r\)th moment.
We fix an entropic regularization parameter \(\varepsilon>0\)
 and consider the quadratic geodesic cost
\begin{equation}
\label{eq-quadratic-cost}
c(x,y):=\frac12 d(x,y)^2,
\qquad x,y\in\mathcal M,
\end{equation}

\paragraph{Entropic OT and the Semidual Formulation.}
For \(\mu,\nu\in\mathcal P(\mathcal M)\) with
\(c\in L^1(\mu\otimes\nu)\), the \emph{entropically regularized optimal
transport problem} is
\begin{equation}
\label{eq:eot-primal-general}
\mathrm{OT}_\varepsilon(\mu,\nu)
:=
\inf_{\pi\in\Pi(\mu,\nu)}
\left\{
\int_{\mathcal M\times\mathcal M} c(x,y)\,d\pi(x,y)
+
\varepsilon\,\mathrm{KL}(\pi\mid \mu\otimes\nu)
\right\}.
\end{equation}
Given a measurable function \(g:\mathcal M\to\mathbb R\), define its
\emph{soft \(c\)-transform relative to \(\nu\)} by
\begin{equation}
\label{eq:soft-c-transform-general}
(\mathcal T_\nu^\varepsilon g)(x)
:=
-\varepsilon\log\!\left(
\int_{\mathcal M}
\exp\!\left(\frac{g(y)-c(x,y)}{\varepsilon}\right)\,d\nu(y)
\right),
\qquad x\in\mathcal M,
\end{equation}
whenever the normalizing integral belongs to \((0,\infty)\). For admissible
\(g\), in the sense made precise in
Proposition~\ref{prop:semidual-reduction-review}, set
\begin{equation}
\label{eq:semidual-functional-general}
\mathcal J_\varepsilon(g)
:=
\int_{\mathcal M} g\,d\nu
+
\int_{\mathcal M}\mathcal T_\nu^\varepsilon g\,d\mu.
\end{equation}
Under the standing assumption \(c\in L^1(\mu\otimes\nu)\), entropic OT admits
Schr\"odinger potentials
\((f_\varepsilon^\star,g_\varepsilon^\star)\), unique up to an additive
constant, satisfying
\[
f_\varepsilon^\star
=
\mathcal T_\nu^\varepsilon g_\varepsilon^\star
\]
\(\mu\)-almost surely, or everywhere on the relevant support after choosing the
canonical pointwise representatives described in
Proposition~\ref{prop:dual-attainment-review}. The unique optimizer
\(\pi_\varepsilon^\star\) has Gibbs density
\begin{equation}
\label{eq:gibbs-density-general}
\frac{d\pi_\varepsilon^\star}{d(\mu\otimes\nu)}(x,y)
=
\exp\!\left(
\frac{f_\varepsilon^\star(x)+g_\varepsilon^\star(y)-c(x,y)}{\varepsilon}
\right).
\end{equation}
Moreover, the problem admits the following one-potential semidual representation, where the supremum is over admissible \(g\):
\[
\mathrm{OT}_\varepsilon(\mu,\nu)
=
\sup_g \mathcal J_\varepsilon(g),
\]

\paragraph{Transport Surrogates Extracted from Gibbs Conditionals.}
Let \(\pi\) be a probability measure on \(\mathcal M\times\mathcal M\) with
first marginal \(\mu\), and write its disintegration as
\[
\pi(dx,dy)=\mu(dx)\,\pi_x(dy).
\]
In general, the conditional law \(\pi_x\) may be viewed as a distributional
transport output associated with the source point \(x\). The two following
intrinsic approaches to extract transport surrogates from these conditionals
will be relevant.

First, when \((\mathcal M,g)\) is Cartan--Hadamard and the conditional laws have
finite second moments, the barycentric projection is defined by
\[
T_\pi(x):=\operatorname{bar}(\pi_x)
:=
\operatorname*{arg\,min}_{z\in\mathcal M}
\frac12\int_{\mathcal M} d(z,y)^2\,\pi_x(dy).
\]
The finite-second-moment condition is automatic in the compact-support setting:
if \(\pi\) is supported on \(\mathcal M\times K_\nu\) for a compact set
\(K_\nu\subset\mathcal M\), then \(\pi_x(K_\nu)=1\) for
\(\mu\)-almost every \(x\), and hence
\(\pi_x\in\mathcal P_2(\mathcal M)\) for \(\mu\)-almost every \(x\).

Second, if \((\mathcal M,g)\) is complete and stochastically complete, with heat
kernel \(p_t(y,z)\), the conditional law \(\pi_x\) in the second variable may be
smoothed by the heat semigroup. For \(t>0\), define
\[
Q_{\pi,t}(x,dz)
:=
\Bigl(\int_{\mathcal M} p_t(y,z)\,\pi_x(dy)\Bigr)\,\mathrm{vol}_g(dz).
\]
Heat smoothing provides a canonical Riemannian way to turn possibly atomic
conditional target laws, such as those obtained from finite samples of a
learned plan, into absolutely continuous conditional distributions. This is
useful when one wants a density-based transport surrogate rather than only a
discrete conditional measure. Depending on the application, the smoothed law
\(Q_{\pi,t}(x,\cdot)\) may be used directly or a point prediction may be derived
from it, for instance by selecting a mode of the heat-smoothed density. As shown
in Proposition~\ref{prop:main-fixed-eps-heat}, the heat-smoothed surrogate is
stable at every fixed heat time \(t>0\), and its population-level smoothing bias
vanishes as \(t\downarrow0\).
\paragraph{Feature-Induced Hypothesis Classes.}
We learn functions on \(D\subset\mathcal M\) by pulling back Euclidean models
through a feature map; see Appendix~\ref{app:feature-map-review} for the
compact-domain transfer principle and intrinsic examples of admissible feature
maps. Let \(D\) be compact, let
\[
\varphi:D\to\mathbb R^n
\]
be a continuous feature map, and let
\(\mathcal F\subset C(\mathbb R^n,\mathbb R)\) be a Euclidean function class,
for example a neural network class. The induced pullback class is
\begin{equation}
\label{eq:pullback-class-general}
\varphi^*\mathcal F
:=
\{a\circ\varphi:\ a\in\mathcal F\}.
\end{equation}

The approximation power of this class depends jointly on \(\mathcal F\) and
\(\varphi\). Since \(D\) is compact, uniform convergence on compact sets
reduces on \(D\) to ordinary uniform convergence. Hence, if \(\mathcal F\) is
dense in \(C(\mathbb R^n,\mathbb R)\) under the ucc topology and \(\varphi\) is
continuous and injective, then \(\varphi^*\mathcal F\) is dense in
\(C(D,\mathbb R)\) under the uniform norm. This applies, for instance, to
standard universal neural network
architectures~\citep{Leshno1993,NIPS2017_32cbf687,Zhou2020} and to posterior
means of Gaussian processes with universal
kernels~\citep{JMLR:v7:micchelli06a}.

Since Schr\"odinger potentials are identifiable only up to additive constants,
we work with the centered pullback class
\begin{equation}
\label{eq-centered-pullback-class}
\mathsf C_\nu(\varphi^*\mathcal F)
:=
\left\{
a\circ\varphi-\int a\circ\varphi\,d\nu
:\ a\in\mathcal F
\right\},
\end{equation}
where \(\mathsf C_\nu h:=h-\int h\,d\nu\) for \(\nu\)-integrable \(h\).

%% file: sections/method.tex
We now present \emph{Entropic RNOT}, the learning framework analyzed in our
work. The model is obtained by restricting the semidual objective to the
centered pullback class introduced in~\eqref{eq-centered-pullback-class}.

\paragraph{Parameterized Semidual Model.}
Let \(\varphi:K_\nu\to\mathbb R^n\) be the target-side feature map introduced
above. Let \(a_\theta\in\mathcal F\) be a Euclidean parametric model, and define
\[
h_\theta:=a_\theta\circ\varphi\in \varphi^*\mathcal F,
\qquad
g_\theta:=\mathsf C_\nu h_\theta.
\]
We use \(g_\theta\) as the target-side Schr\"odinger potential and train the
model by solving
\begin{equation}
\label{eq:training-objective-general}
\max_\theta \mathcal J_\varepsilon(g_\theta).
\end{equation}
Once \(g_\theta\) is given, the corresponding source-side potential is recovered
through the soft \(c\)-transform,
\[
f_\theta^\varepsilon:=\mathcal T_\nu^\varepsilon g_\theta.
\]

\paragraph{Stochastic Optimization.}
In practice, we optimize \(\mathcal J_\varepsilon(g_\theta)\) using minibatch
approximations. Given samples \(x_1,\dots,x_B\sim\mu\) and
\(y_1,\dots,y_B\sim\nu\), we use the empirical centering
\[
g_\theta(y_j)
=
h_\theta(y_j)-\frac1B\sum_{\ell=1}^B h_\theta(y_\ell).
\]
The companion source-side potential is approximated by
\[
f_\theta^\varepsilon(x_i)
=
-\varepsilon\log\!\left(
\frac1B\sum_{j=1}^B
\exp\!\left(
\frac{g_\theta(y_j)-c(x_i,y_j)}{\varepsilon}
\right)
\right)
\]
 and evaluated numerically with a stable log-sum-exp implementation. This yields the following
minibatch objective, which  we maximize by stochastic gradient ascent:
\[
\widehat{\mathcal J}_\varepsilon(\theta)
=
\frac1B\sum_{j=1}^B g_\theta(y_j)
+
\frac1B\sum_{i=1}^B f_\theta^\varepsilon(x_i).
\]

\paragraph{Transport Objects Induced by Entropic RNOT.}
The learned potential \(g_\theta\) induces the Gibbs plan
\begin{equation}
\label{eq:method-neural-gibbs-plan}
d\pi_\theta^\varepsilon(x,y)
=
\exp\!\left(
\frac{f_\theta^\varepsilon(x)+g_\theta(y)-c(x,y)}{\varepsilon}
\right)\,d\mu(x)\,d\nu(y).
\end{equation}
Equivalently, for each source point \(x\), the model defines the conditional law
\[
\pi_{\theta,x}^\varepsilon(dy)
=
\exp\!\left(
\frac{f_\theta^\varepsilon(x) + g_\theta(y)-c(x,y)}{\varepsilon}
\right)\,d\nu(y),
\qquad
\pi_\theta^\varepsilon(dx,dy)
=
\mu(dx)\,\pi_{\theta,x}^\varepsilon(dy).
\]
Thus, the learned model assigns to each source point a regularized conditional
output law \(\pi_{\theta,x}^\varepsilon\). When \((\mathcal M,g)\) is
Cartan--Hadamard and \(\pi_{\theta,x}^\varepsilon\in\mathcal P_2(\mathcal M)\)
for \(\mu\)-almost every \(x\), we denote the induced barycentric surrogate by
\(\widehat T_\theta^\varepsilon:=T_{\pi_\theta^\varepsilon}\). If
\((\mathcal M,g)\) is complete and stochastically complete, we denote the
heat-smoothed surrogate by \(Q_{\theta,t}^\varepsilon:=Q_{\pi_\theta^\varepsilon,t}\)
for \(t>0\).

%% file: sections/theory.tex
We now show that Entropic RNOT is sufficiently expressive to recover
the entropic coupling and the transport surrogates extracted from its Gibbs
conditionals. To keep the argument concrete, we work in the compact-support
regime, which provides a convenient sufficient setting for dual attainment,
uniform approximation of Schr\"odinger potentials, and stability of the induced
Gibbs plans.
Proofs from this section are given in
Appendix~\ref{app-proof-recovery-entropic-plan},
Appendix~\ref{app:proof-main-fixed-eps-map}, and
Appendix~\ref{app:proof-main-fixed-eps-heat}.

Accordingly, throughout this section, we assume that
\(\mu,\nu\in\mathcal P(\mathcal M)\) have compact supports
\[
K_\mu:=\operatorname{spt}(\mu),
\qquad
K_\nu:=\operatorname{spt}(\nu),
\]
that \(\mathcal F\subset C(\mathbb R^n,\mathbb R)\) is dense under the ucc
topology, and that the feature map
\(\varphi:K_\nu\to\mathbb R^n\)
is continuous and injective. Since \(K_\nu\) is compact, \(\varphi(K_\nu)\) is
compact as well, so ucc-density of \(\mathcal F\) implies uniform density on
\(\varphi(K_\nu)\). Consequently, the centered pullback class
\(\mathsf C_\nu(\varphi^*\mathcal F)\) is dense in the centered subspace of
\(C(K_\nu)\) under the uniform norm (see Proposition~\ref{prop:potential-approximation}).

\paragraph{Recovering the Entropic Coupling.}
Our first result below shows that for each fixed \(\varepsilon>0\), the class
\(\mathsf C_\nu(\varphi^*\mathcal F)\) is rich enough to recover the entropic optimal coupling through
the semidual objective.

\begin{theorem}[Recovery of the Entropic Coupling]
\label{thm:main-fixed-eps-plan}
For fixed \(\varepsilon>0\), there exists a sequence
\((g_m)_{m\in\mathbb N}\subset \mathsf C_\nu(\varphi^*\mathcal F)\) such that
\[
\mathcal J_\varepsilon(g_m)\xrightarrow[m\to\infty]{}\mathrm{OT}_\varepsilon.
\]
Let \(\pi_m^\varepsilon\) denote the Gibbs plan induced by \(g_m\). Then,
as \(m\to\infty\),
\[
\mathrm{KL}(\pi_\varepsilon^\star \,\|\, \pi_m^\varepsilon)\to 0,
\qquad
\|\pi_m^\varepsilon-\pi_\varepsilon^\star\|_{\mathrm{TV}}\to 0,
\qquad
\pi_m^\varepsilon\rightharpoonup \pi_\varepsilon^\star .
\]
\end{theorem}

\paragraph{Recovering and Stabilizing Transport Surrogates.}
Beyond recovery of the entropic coupling itself, the same approximation
mechanism also controls the transport surrogates induced by Entropic RNOT and
introduced in Section~\ref{sec:background}. These surrogate results should be understood as
stable consequences of plan recovery. In the Cartan--Hadamard setting, the
entropic barycentric projection is recovered in \(L^2(\mu)\).

\begin{proposition}[Recovery of the Entropic Barycentric Surrogate]
\label{thm:main-fixed-eps-map}
Assume in addition that \((\mathcal M,g)\) is Cartan--Hadamard, and fix
\(\varepsilon>0\). Then there exists a sequence
\((g_m)_{m\in\mathbb N}\subset \mathsf C_\nu(\varphi^*\mathcal F)\) as in
Theorem~\ref{thm:main-fixed-eps-plan}. Let \(\pi_m^\varepsilon\) denote the
Gibbs plan induced by \(g_m\). If
\[
T_\varepsilon:=T_{\pi_\varepsilon^\star},
\qquad
\widehat T_m^\varepsilon:=T_{\pi_m^\varepsilon},
\]
then
\[
\widehat T_m^\varepsilon \to T_\varepsilon
\qquad\text{in }L^2(\mu).
\]
Hence, \(\widehat T_m^\varepsilon\to T_\varepsilon\) also in \(L^1(\mu)\).
\end{proposition}

On complete stochastically complete manifolds, heat smoothing provides a
canonical Riemannian way to turn possibly atomic conditional target laws, such
as those obtained from finite samples or finite target supports of a learned
plan, into absolutely continuous conditional distributions. This is useful when
one wants to work with a density-based transport surrogate rather than only a
discrete conditional law. At the level of the induced smoothed laws, the
heat-smoothed surrogate is stable at every fixed heat time \(t>0\), and its
population-level smoothing bias vanishes as \(t\downarrow0\). For any probability measure \(\pi\) on \(\mathcal M\times\mathcal M\) with
first marginal \(\mu\), written as
\(\pi(dx,dy)=\mu(dx)\pi_x(dy)\), we define for the corresponding heat-smoothed joint measure:
\[
\Pi_{\pi,t}(dx,dz):=\mu(dx)\,Q_{\pi,t}(x,dz)
\]

\begin{proposition}[Stability of the Heat-Smoothed Surrogate]
\label{prop:main-fixed-eps-heat}
Assume in addition that \((\mathcal M,g)\) is complete and stochastically
complete, and fix \(\varepsilon>0\). Let
\((g_m)_{m\in\mathbb N}\subset\mathsf C_\nu(\varphi^*\mathcal F)\) be the
approximating sequence from Theorem~\ref{thm:main-fixed-eps-plan}, with induced
Gibbs plans \(\pi_m^\varepsilon\). For \(t>0\), define the heat-smoothed
joint measures
\[
\Pi_{m,t}^\varepsilon:=\Pi_{\pi_m^\varepsilon,t},
\qquad
\Pi_{\varepsilon,t}^\star:=\Pi_{\pi_\varepsilon^\star,t}.
\]
Then \(Q_{\pi_m^\varepsilon,t}(x,\cdot)\) and
\(Q_{\pi_\varepsilon^\star,t}(x,\cdot)\) are absolutely continuous with respect
to \(\mathrm{vol}_g\) for \(\mu\)-almost every \(x\), and
\[
\|\Pi_{m,t}^\varepsilon-\Pi_{\varepsilon,t}^\star\|_{\mathrm{TV}}
\le
\|\pi_m^\varepsilon-\pi_\varepsilon^\star\|_{\mathrm{TV}}
\xrightarrow[m\to\infty]{}0.
\]
Consequently, for every fixed \(t>0\),
\[
\Pi_{m,t}^\varepsilon \rightharpoonup \Pi_{\varepsilon,t}^\star
\qquad (m\to\infty).
\]
Moreover,
\[
\Pi_{\varepsilon,t}^\star \rightharpoonup \pi_\varepsilon^\star
\qquad (t\downarrow0),
\]
and therefore, for any sequence \(t_m\downarrow0\),
\[
\Pi_{m,t_m}^\varepsilon \rightharpoonup \pi_\varepsilon^\star .
\]
\end{proposition}

In particular, Proposition~\ref{prop:main-fixed-eps-heat} applies to complete
Riemannian manifolds with Ricci curvature bounded from below; see
Appendix~\ref{app:heat-smoothed-stochastic-completeness} for further details.
This includes compact Riemannian manifolds, Euclidean spaces, hyperbolic
spaces, and products such as
\(\mathrm{SE}(3)\simeq \mathrm{SO}(3)\times\mathbb R^3\) equipped with the
product metric used in our experiments. The same appendix also explains why the
criterion applies to the SPD cone endowed with its standard affine-invariant
Riemannian metric.

%% file: sections/experiments.tex
We evaluate Entropic RNOT in three complementary settings. First, we test transport fidelity on controlled intrinsic-geometry benchmarks, measuring how accurately the learned plan and map recover a discrete manifold Sinkhorn reference. Second, we study computational scaling in time, memory, and inference throughput as the support size increases. Third, we demonstrate amortized out-of-sample inference on a real $\mathrm{SE}(3)$ task: crystal-free refinement of protein--ligand docking poses from CrossDocked2020~\cite{Francoeur2020-mc}. Full implementation details, dataset construction, and supplementary qualitative examples are deferred to Appendix~\ref{app:implementation-details}.

\begin{figure}[t]
    \centering
    \includegraphics[width=\textwidth]{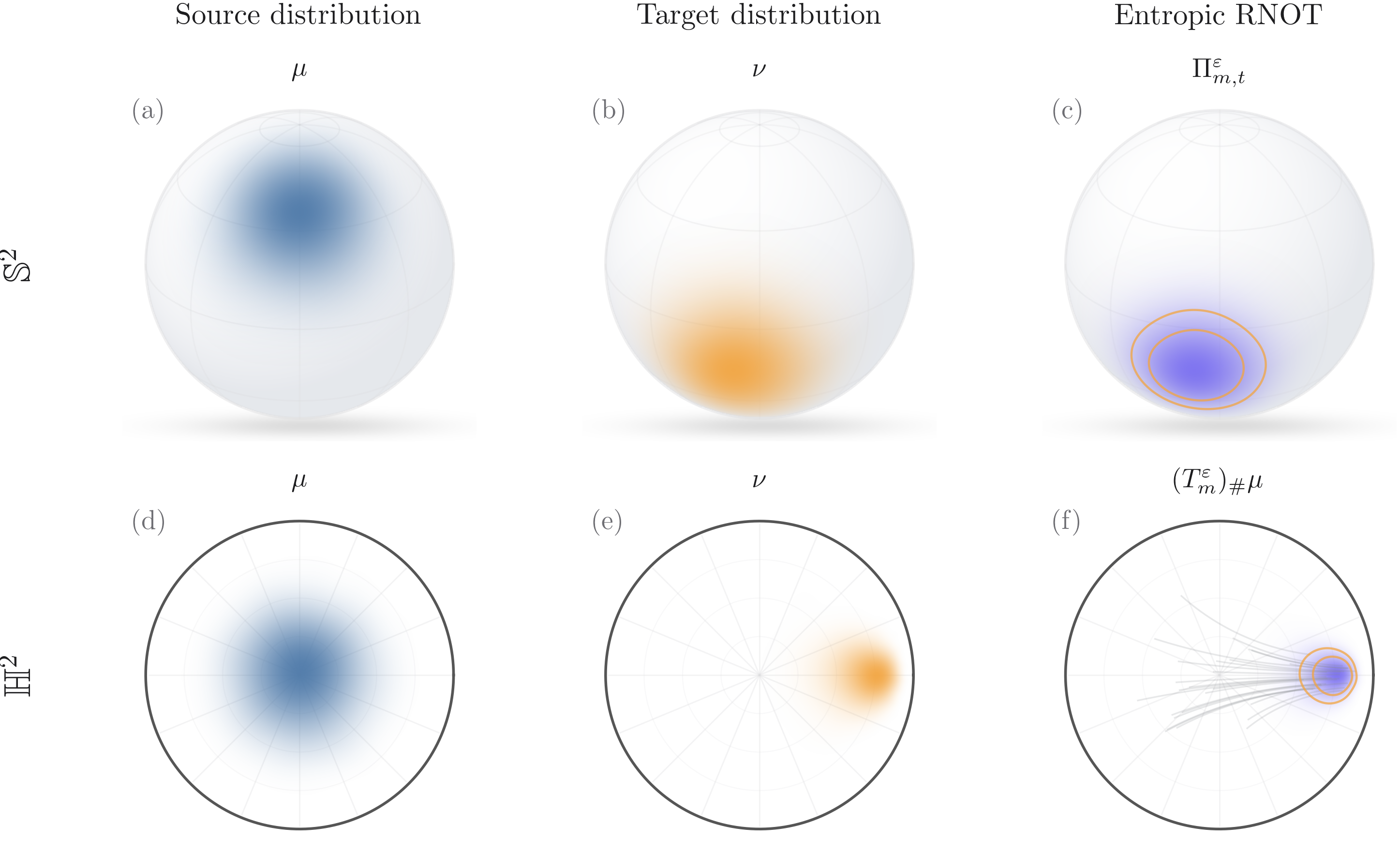}
    \caption{Qualitative transport on $\mathbb{S}^2$ (top) and $\mathbb{H}^2$ (bottom). Each row shows the source distribution $\mu$, the target distribution $\nu$, and the corresponding method output. On $\mathbb{S}^2$, the right panel shows the heat-smoothed output distribution $\hat{\nu}_{\mathrm{heat}}$. On $\mathbb{H}^2$, the right panel shows the pushforward $(\hat T_{\mathrm{bar}})_{\#}\mu$ induced by the barycentric transport map. Orange contours indicate the target distribution, and sparse geodesic curves are shown only in the $\mathbb{H}^2$ output panel to visualize representative barycentric map displacements. In both settings, the method output aligns closely with the target.}
    \label{fig:benchmark_transport}
\end{figure}

\subsection{Intrinsic-Geometry Benchmarks}
\label{sec:exp-synth-geometry}

We first evaluate Entropic RNOT on controlled synthetic benchmarks across five
common geometries: positive-curvature manifolds ($\mathbb{S}^2$, $\mathrm{SO}(3)$), non-positively curved Cartan--Hadamard manifolds ($\mathrm{SPD}(3)$, $\mathbb{H}^2$), and a structured product manifold ($\mathrm{SE}(3)$). Source and target distributions are wrapped normals, with centers chosen to induce realistic geodesic separations in regimes where ambient approximations introduce substantial distortion. The full construction is described in Appendix~\ref{app:implementation-details}.

We compare against two established baselines: ambient Euclidean OT, which ignores the manifold structure entirely, and tangent-space OT, which linearizes the geometry around a single reference point. A discrete manifold Sinkhorn solution, computed directly on the sampled support, serves as the numerical reference. We report plan-level discrepancies from this reference, using KL and conditional $W_1$, together with map-level errors, using ambient $L^2$ and endpoint geodesic error.

Table~\ref{tab:synthetic-benchmarks} summarizes the results; full metrics are reported in Table~\ref{tab:synthetic-manifold-benchmarks}. Our method consistently recovers the reference plan more accurately than all baselines, with the largest gains on $\mathrm{SPD}(3)$, $\mathrm{SE}(3)$, and $\mathbb{H}^2$, where intrinsic geometry plays the strongest role. Appendix~\ref{tab:chnot-vs-rnot-s2} shows that Entropic RNOT also matches non-entropic RNOT in approximation quality on $\mathbb{S}^2$ while being substantially faster and more memory efficient.

\begin{table}[t]
\centering
\small
\caption{Intrinsic-geometry transport benchmarks. Plan KL and endpoint geodesic error are reported relative to a discrete manifold Sinkhorn reference. Lower is better.}
\label{tab:synthetic-benchmarks}
\setlength{\tabcolsep}{4pt}
\begin{tabular}{lccccc}
\toprule
Method & $\mathbb{S}^2$ & $\mathrm{SO}(3)$ & $\mathrm{SPD}(3)$ & $\mathrm{SE}(3)$ & $\mathbb{H}^2$ \\
\midrule
\multicolumn{6}{l}{\textit{Plan KL} $\downarrow$} \\[2pt]
Ambient Euclidean & 0.70 & 1.45 & 1.85 & 1.72 & 0.96 \\
Tangent-space     & 0.48 & 0.15 & 1.23 & 1.32 & 0.10 \\
Entropic RNOT              & \textbf{0.05} & \textbf{0.07} & \textbf{0.01} & \textbf{0.06} & \textbf{0.01} \\
\midrule
\multicolumn{6}{l}{\textit{Endpoint error} $\downarrow$} \\[2pt]
Ambient Euclidean & 0.21 & 0.78 & 0.61 & 0.64 & 0.24 \\
Tangent-space     & 0.21 & 0.25 & 0.43 & 0.51 & 0.11 \\
Entropic RNOT             & \textbf{0.08} & \textbf{0.24} & \textbf{0.04} & \textbf{0.12} & \textbf{0.04} \\
\bottomrule
\end{tabular}
\end{table}

\subsection{Scalability}
\label{sec:exp-scaling}

A primary motivation for our semidual formulation is computational. Discrete manifold Sinkhorn forms and requires a $N \times N$ geodesic cost matrix, incurring $O(N^2)$ memory and super-linear solve time. Our method replaces this with minibatch optimization of a neural parametric potential and batched inference at test time, both independent of the total support size. We compare wall-clock training time, peak GPU memory, and inference throughput as $N$ grows from 128 to 32,768 across all five geometries; the full description is given in Appendix~\ref{app:implementation-details}.

Figure~\ref{fig:scaling} demonstrates the expected computational benefits. Neural training time and memory remain constant in $N$ (${\approx}24$s and ${\approx}11$MB on $\mathbb{S}^2$), driven only by the batch size. Sinkhorn solve time benefits from GPU parallelism\cite{Cuturi2022-ms} and grows sub-quadratically in practice, but its $O(N^2)$ memory footprint is the binding constraint, reaching 34 GB at $N{=}32{,}768$ and exceeding the 36 GB device budget on several manifolds. Inference throughput of the learned potential scales linearly with $N$ (up to a few${\times}10^6$ samples/s), reflecting GPU parallelism over a fixed-cost forward pass. These throughput rates make amortized inference practical: a single trained potential can be evaluated on new source samples at negligible marginal cost, without resolving a discrete OT problem for each query. These results show that our method yields a reusable transport model whose cost is decoupled from support size.

\begin{figure*}[t]
\centering
\includegraphics[width=\textwidth]{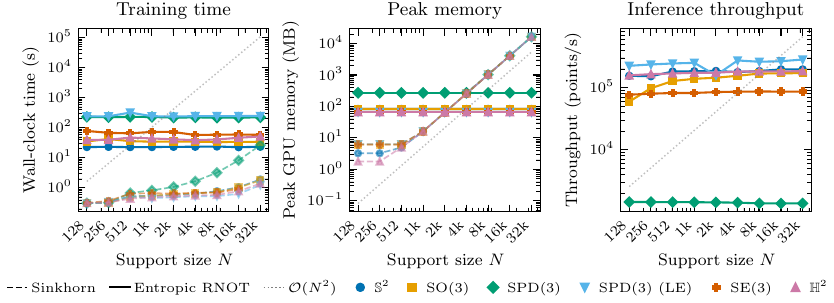}
\caption{Scalability with increasing support size. Left: training wall-clock
time. Center: peak GPU memory. Right: inference throughput. Entropic RNOT
scales more favorably than discrete manifold Sinkhorn, especially in the
large-support regime.}
\label{fig:scaling}
\end{figure*}

\subsection{Real-World Pose Refinement on \texorpdfstring{\(\mathrm{SE}(3)\)}{SE(3)} for Protein--Ligand Docking}
\label{sec:exp-real-se3}

\begin{table}[t]
\centering
\small
\caption{Post-docking pose refinement on $\mathrm{SE}(3)$ (CrossDocked2020, 29 held-out complexes). Per-complex top-1 metrics with 95\% bootstrap CIs. Crystal pose used for evaluation only.}
\label{tab:se3-refinement}
\begin{tabular}{lcccc}
\toprule
Method & RMSD (\AA) $\downarrow$ & Median (\AA) $\downarrow$ & @2\AA\ $\uparrow$ & @5\AA\ $\uparrow$ \\
\midrule
No refinement & 11.24$_{\scriptstyle[8.37,14.14]}$ & 10.39$_{\scriptstyle[6.95,12.39]}$ & 10.3\%$_{\scriptstyle[1.6,20.7]}$ & 24.1\%$_{\scriptstyle[10.3,43.2]}$ \\
GNINA --minimize & 11.03$_{\scriptstyle[8.22,14.89]}$ & 8.68$_{\scriptstyle[4.51,15.27]}$ & 0.0\%$_{\scriptstyle[0.0,0.0]}$ & 38.9\%$_{\scriptstyle[22.2,64.0]}$ \\
Sinkhorn SE(3) & 16.04$_{\scriptstyle[12.15,20.30]}$ & 11.91$_{\scriptstyle[4.29,26.28]}$ & 17.2\%$_{\scriptstyle[6.9,32.8]}$ & 34.5\%$_{\scriptstyle[13.6,50.1]}$ \\
Entropic RNOT & \textbf{3.47$_{\scriptstyle[1.26,6.58]}$} & \textbf{1.42$_{\scriptstyle[1.30,1.87]}$} & \textbf{75.9\%$_{\scriptstyle[60.3,89.7]}$} & \textbf{93.1\%$_{\scriptstyle[86.2,100.0]}$} \\
\bottomrule
\end{tabular}
\end{table}

We evaluate Entropic RNOT on protein--ligand docking poses derived from
CrossDocked2020~\citep{Francoeur2020-mc}, a benchmark of 3{,}765 complexes
across 1{,}302 pocket-similarity clusters. Our goal is to study crystal-free refinement of docking pose ensembles in a rigid-pose setting. We learn an Entropic RNOT refinement map on $\mathrm{SE}(3)$ that moves geometric outliers toward the docking engine's own top-ranked binding basin, thereby acting as a docking-pose refinement or denoising procedure using only quantities available at inference time. Crystallographic poses are never used during training or inference, and are reserved exclusively for held-out evaluation. This setting should be understood as docking-pose refinement rather than full de novo docking or general flexible pose prediction. 

For each complex, we generate a docking-pose ensemble with GNINA~\citep{McNutt2025-np} and convert poses to rigid-body transforms $g_{k,m}\in\mathrm{SE}(3)$ by alignment to a canonical ligand conformer. Source and target pose sets are defined using only the docking engine's own top-ranked pose, and train/test complexes are split by pocket-similarity cluster. We train Entropic RNOT on pooled poses and extract deterministic refinement maps
by heat-smoothed mode finding. Preprocessing choices, thresholds, and hyperparameters are reported in Appendix~\ref{app:app-real-world-experiment}.

Table~\ref{tab:se3-refinement} reports per-complex top-1 metrics on 29 held-out test complexes, evaluated against the crystal pose. Entropic RNOT reduces mean top-1 RMSD from 11.24~\AA\ to 3.47~\AA\ and improves the 2~\AA\ success rate from 10.3\% to 75.9\%; the median top-1 RMSD is 1.42~\AA, indicating near-native refinement for most complexes. A per-complex discrete Sinkhorn baseline on $\mathrm{SE}(3)$, which solves entropic OT separately using each test complex's own source and target poses, performs poorly because the per-complex target sets are too small to identify a stable transport plan. This supports the need for amortized cross-complex learning. GNINA's built-in energy minimization~\citep{McNutt2025-np} yields negligible improvement, with mean RMSD remaining 11.03~\AA, suggesting that local scoring-surface refinement does not provide the global pose redistribution captured by manifold-aware optimal transport. Although this benchmark is restricted to near-rigid poses, we observe similar or better performance on the full unfiltered docking ensembles (Table~\ref{tab:se3-refinement_norigid}).

\begin{figure}
    \centering
    \includegraphics[width=0.9\linewidth]{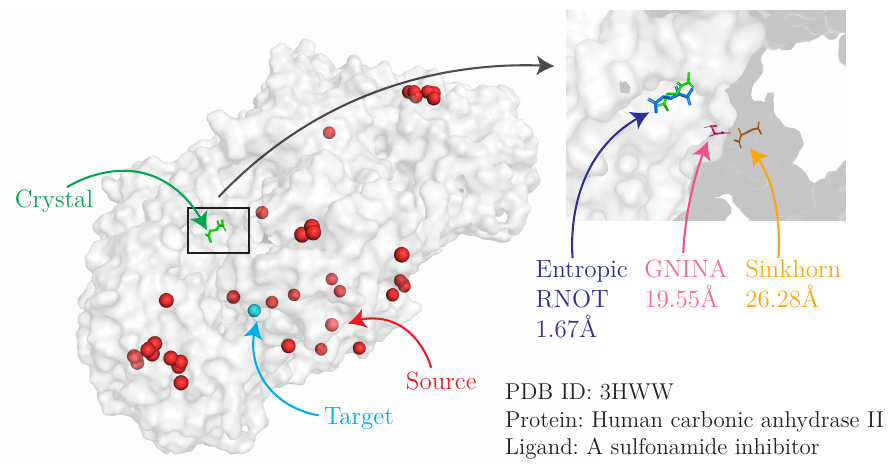}
    \caption{Post-docking pose refinement on complex 3HWW (human carbonic anhydrase II). \textbf{(a)} Ligand centroids from 40 GNINA docked poses: source (red, 39 poses) and target (cyan, 1 pose), with crystal ligand (green). \textbf{(b)} Pocket zoom: the Entropic RNOT refined pose (blue, 1.67~\AA) overlaps the crystal, while GNINA (pink, 19.6~\AA) and per-complex Sinkhorn $\mathrm{SE}(3)$ (orange, 26.3~\AA) land far from the binding site.}
    \label{fig:schematic_dock}
\end{figure}

%% file: sections/discussion.tex
We have introduced a manifold-aware neural framework for entropic optimal transport based on a semidual formulation and pullback parameterizations of Schr\"odinger potentials. Our method learns a target-side potential, recovers the induced Gibbs coupling intrinsically on the manifold, and extracts transport surrogates from its conditional laws: barycentric projections in the
Cartan--Hadamard setting and heat-smoothed conditional laws on both geodesically complete and
stochastically complete manifolds. For fixed \(\varepsilon>0\), our theory proves the ability to recover the entropic coupling, as well as \(L^2\)-recovery of barycentric surrogates when barycenters are uniquely defined, and stability with vanishing population-level bias for heat-smoothed surrogates. Empirically, the framework performs well across synthetic manifold benchmarks, scales better than
discrete manifold Sinkhorn, and is shown to be effective in practice for real \(\mathrm{SE}(3)\) protein--ligand pose refinement.

At positive temperature, however, the canonical transport object is a \emph{coupling}, not a deterministic map. Barycentric projections and heat-smoothed modes should therefore be viewed as summaries of the learned
conditional law, useful when point-valued outputs are required, rather than as replacements for the entropic coupling. In the docking experiment, this means
that our method is a refinement or denoising procedure for an existing ensemble of docked poses, not an end-to-end docking or generative pose-prediction model. Specifically, it requires candidate poses from an external docking tool and refines them toward a target basin defined by docking-generated geometry and scores. Notably, this target construction does not require crystallographic supervision.

Several limitations remain. The theory treats fixed \(\varepsilon>0\), compact supports, and approximation assumptions on the feature map and Euclidean
hypothesis class; it does not address the vanishing-regularization regime. Barycentric map guarantees require the Cartan--Hadamard setting, while outside
it barycenters may be nonunique or unstable. Computationally, performance depends on efficient geodesic-distance evaluation, stable log-sum-exp computations, and expressive features. In docking, the current unconditional model ignores receptor pocket context and the sequence and represents poses only as rigid motions on \(\mathrm{SE}(3)\), omitting torsional flexibility. Pocket-conditioned transport, product-manifold pose representations with torsions, and broader evaluation across docking benchmarks are natural directions for future work.

%% file: sections/appendix.tex
\section{Review on Entropic Riemannian Optimal Transport}
\label{app:review-eot}
\input{sections/review-riemannian-OT}
\clearpage
\newpage

\section{Review on Geometric Deep Learning}
\label{app:feature-map-review}
\input{sections/review-geometric-dl}
\clearpage
\newpage


\section{Implementation Details}
\label{app:implementation-details}
\input{sections/implementation-details}

\clearpage
\newpage

\section{Algorithms}
\label{app:algorithms}
\input{sections/algorithms}
\clearpage
\newpage

\section{Additional Experimental Results}
\label{app:additional-experimental-results}
\input{sections/app-additional-experimental-results}
\clearpage
\newpage

\section{Proofs}
\input{sections/proofs}

%% file: sections/review-riemannian-OT.tex
Throughout, \((\mathcal M,g)\) denotes a complete, possibly noncompact,
\(p\)-dimensional Riemannian manifold with geodesic distance \(d\). We write
\(\mathcal P(\mathcal M)\) for the Borel probability measures on \(\mathcal M\)
and \(\mathrm{vol}_{\mathcal M}\) for the Riemannian volume measure. Let
\(\mu,\nu\in\mathcal P(\mathcal M)\) be Borel probability measures and define
the quadratic geodesic cost
\[
c(x,y):=\frac12 d(x,y)^2,
\qquad x,y\in\mathcal M.
\]
For the entropic duality results recalled below, the standing assumption is
\[
c\in L^1(\mu\otimes\nu).
\]
This is the integrability condition used in the general theory of entropic
optimal transport; see \cite[Theorem~3.2]{nutz2022entropicot}. In the
compact-support setting, this condition is automatic. Indeed, if
\[
K_\mu:=\operatorname{spt}(\mu),
\qquad
K_\nu:=\operatorname{spt}(\nu)
\]
are compact, then \(K_\mu\times K_\nu\) is compact and \(c\) is continuous,
hence bounded, on \(K_\mu\times K_\nu\).

Whenever compact supports are assumed, every coupling
\(\pi\in\Pi(\mu,\nu)\) is supported on \(K_\mu\times K_\nu\), so the transport
problem may be localized to this compact product set. Since the geodesic
distance \(d\) is continuous on \(\mathcal M\times\mathcal M\), the quadratic
cost \(c(x,y)=\frac12 d(x,y)^2\) is continuous. Hence, by the Heine--Cantor
theorem, \(c\) is bounded and uniformly continuous on \(K_\mu\times K_\nu\).
Boundedness gives \(c\in L^1(\mu\otimes\nu)\), while uniform continuity is the
regularity assumption used below to obtain bounded uniformly continuous
Schr\"odinger potentials.

\subsection{Entropic primal, dual, and Schr\"odinger potentials}

Fix \(\varepsilon>0\). The entropically regularized transport problem is
\begin{equation}
\label{eq:eot-primal-review}
\mathrm{OT}_\varepsilon
:=
\inf_{\pi\in \Pi(\mu,\nu)}
\left\{
\int_{\mathcal M\times\mathcal M} c(x,y)\,d\pi(x,y)
+
\varepsilon\,\mathrm{KL}(\pi\mid \mu\otimes \nu)
\right\}.
\end{equation}
If \(\mu\) and \(\nu\) have compact supports \(K_\mu\) and \(K_\nu\), then this
is equivalently
\begin{equation}
\label{eq:eot-compact-primal-review}
\mathrm{OT}_\varepsilon
=
\inf_{\pi\in \Pi(\mu,\nu)}
\left\{
\int_{K_\mu\times K_\nu} c(x,y)\,d\pi(x,y)
+
\varepsilon\,\mathrm{KL}(\pi\mid \mu\otimes \nu)
\right\},
\end{equation}
because every coupling is concentrated on \(K_\mu\times K_\nu\).

For measurable \(g:\mathcal M\to\mathbb R\), define the soft \(c\)-transform
\begin{equation}
\label{eq:soft-c-transform-nu-review}
(\mathcal T_\nu^\varepsilon g)(x)
:=
-\varepsilon\log\left(
\int_{\mathcal M}
\exp\!\left(\frac{g(y)-c(x,y)}{\varepsilon}\right)\,d\nu(y)
\right),
\qquad x\in \mathcal M.
\end{equation}
Similarly, for measurable \(f:\mathcal M\to\mathbb R\), define
\begin{equation}
\label{eq:soft-c-transform-mu-review}
(\mathcal T_\mu^\varepsilon f)(y)
:=
-\varepsilon\log\left(
\int_{\mathcal M}
\exp\!\left(\frac{f(x)-c(x,y)}{\varepsilon}\right)\,d\mu(x)
\right),
\qquad y\in \mathcal M.
\end{equation}
These transforms characterize the optimal Schr\"odinger potentials through the
Schr\"odinger system
\[
f_\varepsilon^\star=\mathcal T_\nu^\varepsilon g_\varepsilon^\star,
\qquad
g_\varepsilon^\star=\mathcal T_\mu^\varepsilon f_\varepsilon^\star,
\]
stated precisely in Proposition~\ref{prop:dual-attainment-review}; see also
\cite[Section~4.1]{nutz2022entropicot}. When \(\mu\) and \(\nu\) are compactly
supported, only the restrictions of these functions to \(K_\mu\) and \(K_\nu\)
are relevant. In that case, the transforms can be written equivalently as
integrals over \(K_\nu\) and \(K_\mu\), respectively.

We use the standard pointwise representatives of the Schr\"odinger potentials.
The dual maximizers are initially defined only up to \(\mu\)- and
\(\nu\)-almost-sure equivalence. However, the Schr\"odinger system provides
canonical pointwise versions: each potential can be redefined by the
right-hand side of its soft \(c\)-transform identity, after which the system
holds everywhere on the relevant supports; see
\cite[Section~4.1]{nutz2022entropicot}. This pointwise choice is important
below because our approximation arguments take place in \(C(K_\nu)\) with the
uniform norm, rather than only in an almost-sure equivalence class.

In the compact-support setting, these pointwise representatives are regular.
Since \(c\) is bounded and uniformly continuous on \(K_\mu\times K_\nu\), the
regularity estimates for entropic OT potentials imply that the
\(\varepsilon\)-scaled Schr\"odinger potentials may be chosen bounded and
uniformly continuous on \(K_\mu\) and \(K_\nu\); see
\cite[Lemmas~4.9 and~4.11, Remark~4.12]{nutz2022entropicot}. In particular, the
normalized target-side potential used below may be chosen in \(C(K_\nu)\).

The next statement recalls the standard entropic OT duality, the associated
Schr\"odinger system, and the Gibbs form of the optimal coupling.

\begin{proposition}[Entropic duality, Schr\"odinger system, and Gibbs form]
\label{prop:dual-attainment-review}
Assume that \(c\in L^1(\mu\otimes\nu)\). Then for every \(\varepsilon>0\), the
problem \eqref{eq:eot-primal-review} admits a unique minimizer
\(\pi_\varepsilon^\star\in\Pi(\mu,\nu)\). Moreover, there exist measurable
functions
\[
f_\varepsilon^\star:\mathcal M\to\mathbb R,
\qquad
g_\varepsilon^\star:\mathcal M\to\mathbb R,
\]
unique up to the transformation
\[
(f_\varepsilon^\star,g_\varepsilon^\star)
\mapsto
(f_\varepsilon^\star+a,g_\varepsilon^\star-a),
\qquad a\in\mathbb R,
\]
such that
\begin{equation}
\label{eq:schroedinger-system-review}
f_\varepsilon^\star=\mathcal T_\nu^\varepsilon g_\varepsilon^\star
\quad \mu\text{-a.s.},
\qquad
g_\varepsilon^\star=\mathcal T_\mu^\varepsilon f_\varepsilon^\star
\quad \nu\text{-a.s.},
\end{equation}
and
\begin{equation}
\label{eq:gibbs-density-review}
\frac{d\pi_\varepsilon^\star}{d(\mu\otimes \nu)}(x,y)
=
\exp\!\left(
\frac{f_\varepsilon^\star(x)+g_\varepsilon^\star(y)-c(x,y)}{\varepsilon}
\right)
\qquad
(\mu\otimes \nu)\text{-a.s.}
\end{equation}
Equivalently,
\begin{equation}
\label{eq:dual-eot-review}
\mathrm{OT}_\varepsilon
=
\sup_{f,g}
\left\{
\int f\,d\mu+\int g\,d\nu
-
\varepsilon
\iint
\exp\!\left(\frac{f(x)+g(y)-c(x,y)}{\varepsilon}\right)
\,d\mu(x)d\nu(y)
+\varepsilon
\right\},
\end{equation}
where the supremum is taken over admissible measurable pairs and is attained at
\((f_\varepsilon^\star,g_\varepsilon^\star)\).

If, in addition, \(\mu\) and \(\nu\) have compact supports \(K_\mu\) and
\(K_\nu\), then the optimal potentials admit bounded uniformly continuous
representatives on \(K_\mu\) and \(K_\nu\). In particular, after fixing the
additive constant by a normalization such as
\[
\int_{K_\nu} g_\varepsilon^\star\,d\nu=0,
\]
the normalized target-side potential may be chosen in \(C(K_\nu)\).
\end{proposition}

\begin{proof}
Define the static Schr\"odinger reference measure
\[
R_\varepsilon(dx,dy)
:=
Z_\varepsilon^{-1}
\exp\!\left(-\frac{c(x,y)}{\varepsilon}\right)\mu(dx)\nu(dy),
\qquad
Z_\varepsilon
:=
\iint
\exp\!\left(-\frac{c(x,y)}{\varepsilon}\right)
\,d\mu(x)d\nu(y).
\]
Then, for every \(\pi\in\Pi(\mu,\nu)\),
\[
\int c\,d\pi+\varepsilon\,\mathrm{KL}(\pi\mid\mu\otimes\nu)
=
\varepsilon\,\mathrm{KL}(\pi\mid R_\varepsilon)
-\varepsilon\log Z_\varepsilon,
\]
with both sides interpreted as \(+\infty\) when
\(\pi\not\ll\mu\otimes\nu\). Thus \eqref{eq:eot-primal-review} is equivalent,
up to an additive constant, to the static Schr\"odinger problem with reference
measure \(R_\varepsilon\).

The existence, uniqueness, dual attainment, Schr\"odinger system, and Gibbs
factorization then follow from the entropic optimal transport duality theorem
\cite[Theorem~4.7]{nutz2022entropicot}, applied to the cost \(c/\varepsilon\)
or, equivalently, to the reference measure \(R_\varepsilon\). In the
compact-support case, \(c\) is bounded and uniformly continuous on
\(K_\mu\times K_\nu\). The boundedness and uniform-continuity estimates for the
\(\varepsilon\)-scaled potentials follow from
\cite[Lemmas~4.9 and~4.11, Remark~4.12]{nutz2022entropicot}.
\end{proof}
\subsection{One-Potential Semidual Formulation}

Define, for measurable \(g:\mathcal M\to\mathbb R\),
\begin{equation}
\label{eq:semidual-functional-review}
\mathcal J_\varepsilon(g)
:=
\int_{\mathcal M} g\,d\nu
+
\int_{\mathcal M} \mathcal T_\nu^\varepsilon g\,d\mu.
\end{equation}
In the compact-support setting, this reduces to
\[
\mathcal J_\varepsilon(g)
=
\int_{K_\nu} g\,d\nu
+
\int_{K_\mu} \mathcal T_\nu^\varepsilon g\,d\mu.
\]

The regularized semidual formulation is standard in entropic optimal transport;
see \cite[Section~2.2]{CuturiPeyre2018SemiDual} and
\cite[Section~5.3]{PeyreCuturi2019}. We now include a short derivation to fix
notation and to make explicit the normalization used in
\eqref{eq:dual-eot-review}.

\begin{proposition}[Semidual formulation]
\label{prop:semidual-reduction-review}
Assume that \(c\in L^1(\mu\otimes\nu)\). For every measurable
\(g:\mathcal M\to\mathbb R\) such that
\[
g\in L^1(\nu),
\qquad
0<Z_g(x)<\infty\quad\text{for }\mu\text{-a.e. }x,
\qquad
\mathcal T_\nu^\varepsilon g\in L^1(\mu),
\]
where
\[
Z_g(x)
:=
\int_{\mathcal M}
\exp\!\left(\frac{g(y)-c(x,y)}{\varepsilon}\right)\,d\nu(y),
\]
one has
\[
\mathcal J_\varepsilon(g)\le \mathrm{OT}_\varepsilon.
\]
Moreover,
\[
\mathrm{OT}_\varepsilon
=
\sup_g \mathcal J_\varepsilon(g),
\]
where the supremum is taken over such admissible \(g\). If
\((f_\varepsilon^\star,g_\varepsilon^\star)\) is a maximizing dual pair from
Proposition~\ref{prop:dual-attainment-review}, then
\[
f_\varepsilon^\star
=
\mathcal T_\nu^\varepsilon g_\varepsilon^\star
\qquad \mu\text{-a.s.}
\]
For the canonical pointwise representatives on compact supports, this identity
holds everywhere on \(K_\mu\).
\end{proposition}

\begin{proof}
Let \(g:\mathcal M\to\mathbb R\) be admissible. Define
\[
Z_g(x)
:=
\int_{\mathcal M}
\exp\!\left(\frac{g(y)-c(x,y)}{\varepsilon}\right)\,d\nu(y).
\]
By admissibility,
\[
0<Z_g(x)<\infty
\qquad
\mu\text{-a.e.}
\]
and by definition of the soft \(c\)-transform,
\[
(\mathcal T_\nu^\varepsilon g)(x)
=
-\varepsilon\log Z_g(x)
\qquad
\mu\text{-a.e.}
\]

We first identify the best choice of \(f\) for this fixed \(g\). For a
measurable \(f:\mathcal M\to\mathbb R\), define the two-potential dual
objective from \eqref{eq:dual-eot-review} by
\[
\begin{aligned}
D_\varepsilon(f,g)
&:=
\int_{\mathcal M} f\,d\mu+\int_{\mathcal M} g\,d\nu
-
\varepsilon
\iint
\exp\!\left(\frac{f(x)+g(y)-c(x,y)}{\varepsilon}\right)
\,d\mu(x)d\nu(y)
+\varepsilon .
\end{aligned}
\]
For fixed \(g\), this can be written as
\[
D_\varepsilon(f,g)
=
\int_{\mathcal M} g\,d\nu
+
\int_{\mathcal M}
\Psi_x(f(x))\,d\mu(x),
\]
where, for each fixed \(x\),
\[
\Psi_x(a)
:=
a
-
\varepsilon
\int_{\mathcal M}
\exp\!\left(\frac{a+g(y)-c(x,y)}{\varepsilon}\right)\,d\nu(y)
+
\varepsilon.
\]
Using the definition of \(Z_g(x)\), we have
\[
\int_{\mathcal M}
\exp\!\left(\frac{a+g(y)-c(x,y)}{\varepsilon}\right)\,d\nu(y)
=
e^{a/\varepsilon}Z_g(x).
\]
Therefore
\[
\Psi_x(a)
=
a-\varepsilon e^{a/\varepsilon}Z_g(x)+\varepsilon.
\]
For \(\mu\)-a.e. \(x\), this is a strictly concave function of \(a\), because
\[
\Psi_x''(a)
=
-\frac{1}{\varepsilon}e^{a/\varepsilon}Z_g(x)<0.
\]
Its derivative is
\[
\Psi_x'(a)
=
1-e^{a/\varepsilon}Z_g(x).
\]
Hence the unique maximizer is characterized by
\[
e^{a/\varepsilon}Z_g(x)=1.
\]
Equivalently, the unique maximizer is
\[
a^\star
:=
-\varepsilon\log Z_g(x)
=
(\mathcal T_\nu^\varepsilon g)(x).
\]
Thus, for \(\mu\)-a.e. \(x\),
\[
\Psi_x(a)
\le
\Psi_x(a^\star)
\qquad
\text{for all }a\in\mathbb R.
\]

It remains to compute the value of \(\Psi_x\) at this maximizer. Since
\[
e^{a^\star/\varepsilon}Z_g(x)=1,
\]
we have
\[
\int_{\mathcal M}
\exp\!\left(
\frac{a^\star+g(y)-c(x,y)}{\varepsilon}
\right)\,d\nu(y)
=
1.
\]
Therefore
\[
\begin{aligned}
\Psi_x(a^\star)
&=
a^\star
-
\varepsilon
\int_{\mathcal M}
\exp\!\left(
\frac{a^\star+g(y)-c(x,y)}{\varepsilon}
\right)\,d\nu(y)
+
\varepsilon \\
&=
a^\star-\varepsilon+\varepsilon \\
&=
a^\star \\
&=
(\mathcal T_\nu^\varepsilon g)(x).
\end{aligned}
\]
Consequently, for every measurable \(f\) for which the dual objective is
well defined,
\[
D_\varepsilon(f,g)
\le
\int_{\mathcal M} g\,d\nu
+
\int_{\mathcal M}\mathcal T_\nu^\varepsilon g\,d\mu
=
\mathcal J_\varepsilon(g),
\]
and equality is attained by choosing
\[
f=\mathcal T_\nu^\varepsilon g.
\]

Since \(D_\varepsilon(f,g)\le \mathrm{OT}_\varepsilon\) for every admissible
dual pair \((f,g)\), we may choose
\(f=\mathcal T_\nu^\varepsilon g\) and obtain
\[
\mathcal J_\varepsilon(g)
=
D_\varepsilon(\mathcal T_\nu^\varepsilon g,g)
\le
\mathrm{OT}_\varepsilon.
\]
Taking the supremum over admissible \(g\) gives
\[
\sup_g \mathcal J_\varepsilon(g)
\le
\mathrm{OT}_\varepsilon.
\]

Conversely, let
\((f_\varepsilon^\star,g_\varepsilon^\star)\) be a maximizing dual pair from
Proposition~\ref{prop:dual-attainment-review}. By the Schr\"odinger system
\eqref{eq:schroedinger-system-review},
\[
f_\varepsilon^\star
=
\mathcal T_\nu^\varepsilon g_\varepsilon^\star
\qquad
\mu\text{-a.s.}
\]
Moreover, this identity implies
\[
\int_{\mathcal M}
\exp\!\left(
\frac{
f_\varepsilon^\star(x)+g_\varepsilon^\star(y)-c(x,y)
}{\varepsilon}
\right)\,d\nu(y)
=
1
\qquad
\mu\text{-a.e. }x.
\]
Substituting the optimal pair into the dual formula
\eqref{eq:dual-eot-review}, the exponential term becomes
\[
-\varepsilon
\iint
\exp\!\left(
\frac{
f_\varepsilon^\star(x)+g_\varepsilon^\star(y)-c(x,y)
}{\varepsilon}
\right)
\,d\nu(y)d\mu(x).
\]
By the preceding normalization, the inner integral equals \(1\) for
\(\mu\)-a.e. \(x\). Therefore
\[
-\varepsilon
\iint
\exp\!\left(
\frac{
f_\varepsilon^\star(x)+g_\varepsilon^\star(y)-c(x,y)
}{\varepsilon}
\right)
\,d\nu(y)d\mu(x)
=
-\varepsilon\int_{\mathcal M}1\,d\mu(x)
=
-\varepsilon,
\]
because \(\mu\) is a probability measure. This cancels with the
\(+\varepsilon\) in the dual normalization. Hence
\[
\begin{aligned}
\mathrm{OT}_\varepsilon
&=
D_\varepsilon(f_\varepsilon^\star,g_\varepsilon^\star) \\
&=
\int_{\mathcal M} f_\varepsilon^\star\,d\mu
+
\int_{\mathcal M} g_\varepsilon^\star\,d\nu \\
&=
\int_{\mathcal M}
\mathcal T_\nu^\varepsilon g_\varepsilon^\star\,d\mu
+
\int_{\mathcal M}
g_\varepsilon^\star\,d\nu \\
&=
\mathcal J_\varepsilon(g_\varepsilon^\star).
\end{aligned}
\]
Thus
\[
\mathrm{OT}_\varepsilon
\le
\sup_g \mathcal J_\varepsilon(g).
\]
Together with the opposite inequality proved above, this gives
\[
\mathrm{OT}_\varepsilon
=
\sup_g \mathcal J_\varepsilon(g).
\]
Finally,
\[
f_\varepsilon^\star
=
\mathcal T_\nu^\varepsilon g_\varepsilon^\star
\qquad
\mu\text{-a.s.}
\]
by the Schr\"odinger system. For the canonical pointwise representatives on
compact supports, this identity holds everywhere on \(K_\mu\).
\end{proof}
\subsection{Heat Kernels and Stochastic Completeness}
\label{app:heat-smoothed-stochastic-completeness}

We recall the heat-kernel facts used in
Proposition~\ref{prop:main-fixed-eps-heat}. Grigor'yan treats the more general
case of a weighted manifold \((M,g,\mu)\). In our setting we take
\(\mu=\mathrm{vol}_{\mathcal M}\), the Riemannian volume measure.

\paragraph{Heat semigroup and heat kernel.}
Let \((P_t)_{t>0}\) denote the minimal heat semigroup associated with the
Laplace--Beltrami operator on \((\mathcal M,g)\). We use the heat kernel
\(p_t(x,y)\) in the sense of \cite[Definition~7.12]{Grigoryan2009}. The
properties needed below are those collected in
\cite[Theorem~7.13]{Grigoryan2009}. Namely, for every \(f\in L^2(\mathcal M)\),
every \(x\in\mathcal M\), and every \(t>0\),
\[
P_t f(x)
=
\int_{\mathcal M} p_t(x,y)f(y)\,\mathrm{vol}_{\mathcal M}(dy).
\]
Moreover,
\[
p_t(x,y)=p_t(y,x),
\qquad
p_t(x,y)\ge 0,
\]
and
\[
\int_{\mathcal M}p_t(x,y)\,\mathrm{vol}_{\mathcal M}(dy)\le 1.
\]
Thus the minimal heat kernel is sub-Markovian in general. The same theorem also
gives the semigroup identity
\[
p_{t+s}(x,y)
=
\int_{\mathcal M}p_t(x,z)p_s(z,y)\,
\mathrm{vol}_{\mathcal M}(dz),
\qquad s,t>0,
\]
and the heat-equation regularity property: for fixed \(y\in\mathcal M\), the
function
\[
u(t,x):=p_t(x,y)
\]
is smooth on \((0,\infty)\times\mathcal M\) and satisfies the heat equation.

Finally, \cite[Theorem~7.13]{Grigoryan2009} states the small-time
initial-condition property for compactly supported smooth functions: for every
\(f\in C_c^\infty(\mathcal M)\),
\[
\int_{\mathcal M}p_t(x,y)f(y)\,\mathrm{vol}_{\mathcal M}(dy)
\xrightarrow[t\downarrow0]{}
f(x),
\]
with convergence in \(C^\infty(\mathcal M)\).

\paragraph{Markovian inequalities.}
We also use the Markovian properties of the heat semigroup. In the notation of
\cite[Ch.~5, Section~5.3]{Grigoryan2009}, the heat semigroup satisfies the
following inequalities: if \(f\ge 0\), then
\[
P_t f\ge 0,
\]
and if \(f\le 1\), then
\[
P_t f\le 1.
\]
Consequently, if \(0\le f\le 1\), then
\[
0\le P_t f\le 1.
\]
This is the form used in the proof of
Proposition~\ref{prop:main-fixed-eps-heat} to show that the heat-kernel
averaging of an indicator function remains bounded between \(0\) and \(1\).

\paragraph{Stochastic completeness and conservative Markov kernels.}
Following \cite[Definition~8.17]{Grigoryan2009}, a weighted manifold is called
stochastically complete if its heat kernel satisfies
\[
\int_{\mathcal M}p_t(x,y)\,\mathrm{vol}_{\mathcal M}(dy)=1
\qquad
\text{for all }x\in\mathcal M,\ t>0.
\]
Combining stochastic completeness with the nonnegativity
\(p_t(x,y)\ge 0\) from \cite[Theorem~7.13]{Grigoryan2009}, we obtain that
\[
P_t(x,dy):=p_t(x,y)\,\mathrm{vol}_{\mathcal M}(dy)
\]
is a conservative Markov kernel for every \(t>0\). In particular,
\[
P_t(x,\mathcal M)=1,
\]
and for every bounded measurable function \(h\),
\[
\|P_t h\|_\infty\le \|h\|_\infty.
\]

\paragraph{Extension of the small-time limit to \(C_b(\mathcal M)\).}
The proof of Proposition~\ref{prop:main-fixed-eps-heat} uses the small-time
limit of the heat semigroup on bounded continuous test functions. Grigor'yan's
heat-kernel theorem gives this limit first for compactly supported smooth
functions; see \cite[Theorem~7.13]{Grigoryan2009}. We record the standard
extension.
\begin{lemma}[Small-time continuity on bounded continuous functions]
\label{lem:heat-small-time-Cb}
Assume that \((\mathcal M,g)\) is stochastically complete, and let
\[
P_t(x,dy):=p_t(x,y)\,\mathrm{vol}_{\mathcal M}(dy)
\]
be the heat-kernel Markov kernel. Then, for every \(x\in\mathcal M\) and every
\(h\in C_b(\mathcal M)\),
\[
P_t h(x)
=
\int_{\mathcal M}h(y)\,P_t(x,dy)
\xrightarrow[t\downarrow0]{}
h(x).
\]
\end{lemma}

\begin{proof}
Fix \(x\in\mathcal M\), and write
\[
\nu_t^x(dy):=P_t(x,dy)
=
p_t(x,y)\,\mathrm{vol}_{\mathcal M}(dy).
\]
By \cite[Theorem~7.13]{Grigoryan2009}, for every
\(f\in C_c^\infty(\mathcal M)\),
\[
P_t f(x)
=
\int_{\mathcal M} f(y)\,\nu_t^x(dy)
\xrightarrow[t\downarrow0]{}
f(x).
\]

We first extend this convergence from \(C_c^\infty(\mathcal M)\) to
\(C_c(\mathcal M)\). Let \(f\in C_c(\mathcal M)\), and let \(\delta>0\).
Since \((\mathcal M,g)\) is a Riemannian manifold, \(\mathcal M\) is smooth,
and \(C_c^\infty(\mathcal M)\) is uniformly dense in \(C_c(\mathcal M)\);
this follows from
the locally compact Stone--Weierstrass theorem. Hence there exists
\(g\in C_c^\infty(\mathcal M)\) such that
\[
\|f-g\|_\infty<\delta.
\]
Using the sub-Markov property of the heat semigroup,
\[
|P_t(f-g)(x)|
\le
\|f-g\|_\infty.
\]
Therefore
\[
\begin{aligned}
|P_t f(x)-f(x)|
&\le
|P_t(f-g)(x)|
+
|P_tg(x)-g(x)|
+
|g(x)-f(x)| \\
&\le
2\delta
+
|P_tg(x)-g(x)|.
\end{aligned}
\]
Since \(g\in C_c^\infty(\mathcal M)\), there exists \(t_0>0\) such that for
all \(0<t<t_0\),
\[
|P_tg(x)-g(x)|<\delta.
\]
Therefore, for all \(0<t<t_0\),
\[
|P_t f(x)-f(x)|
\le
3\delta.
\]
Since \(\delta>0\) was arbitrary, \(P_t f(x)\to f(x)\).

We now prove the claim for \(h\in C_b(\mathcal M)\). Since \(\mathcal M\) is a smooth manifold, smooth bump functions exist; hence
there is \(\chi\in C_c^\infty(\mathcal M)\subset C_c(\mathcal M)\) such that
\[
0\le \chi\le 1,
\qquad
\chi(x)=1,
\]
see \cite[Proposition~2.25]{LeeSmoothManifolds}.
Then \(h\chi\in C_c(\mathcal M)\), so the \(C_c\)-convergence just proved gives
\[
\int_{\mathcal M}h(y)\chi(y)\,\nu_t^x(dy)
\xrightarrow[t\downarrow0]{}
h(x)\chi(x)=h(x).
\]
Likewise, since \(\chi\in C_c(\mathcal M)\),
\[
\int_{\mathcal M}\chi(y)\,\nu_t^x(dy)
\xrightarrow[t\downarrow0]{}
\chi(x)=1.
\]
By stochastic completeness,
\[
\nu_t^x(\mathcal M)=P_t(x,\mathcal M)=1.
\]
Therefore
\[
\int_{\mathcal M}(1-\chi(y))\,\nu_t^x(dy)
=
1-\int_{\mathcal M}\chi(y)\,\nu_t^x(dy)
\xrightarrow[t\downarrow0]{}
0.
\]
Now decompose
\[
P_t h(x)
=
\int_{\mathcal M}h(y)\chi(y)\,\nu_t^x(dy)
+
\int_{\mathcal M}h(y)(1-\chi(y))\,\nu_t^x(dy).
\]
Hence
\[
\begin{aligned}
|P_t h(x)-h(x)|
&\le
\left|
\int_{\mathcal M}h(y)\chi(y)\,\nu_t^x(dy)-h(x)
\right| \\
&\quad+
\int_{\mathcal M}|h(y)|(1-\chi(y))\,\nu_t^x(dy) \\
&\le
\left|
\int_{\mathcal M}h(y)\chi(y)\,\nu_t^x(dy)-h(x)
\right|
+
\|h\|_\infty
\int_{\mathcal M}(1-\chi(y))\,\nu_t^x(dy).
\end{aligned}
\]
Both terms on the right-hand side tend to \(0\). Therefore
\[
P_t h(x)
=
\int_{\mathcal M}h(y)\,P_t(x,dy)
\xrightarrow[t\downarrow0]{}
h(x),
\]
as claimed.
\end{proof}

\paragraph{Geometric criteria for stochastic completeness.}
We now recall standard geometric criteria ensuring the stochastic-completeness
assumption used above. Let \((\mathcal M,g)\) be a geodesically complete
\(d\)-dimensional Riemannian manifold, and write
\[
V(x_0,r):=\mathrm{vol}_{\mathcal M}(B(x_0,r)).
\]
A volume-growth criterion of \cite[Theorem~9.1]{Grigoryan1999} states that if,
for some \(x_0\in\mathcal M\),
\[
\int^\infty \frac{r\,dr}{\log V(x_0,r)}=\infty,
\]
then \((\mathcal M,g)\) is stochastically complete. See also
\cite[Ch.~11, Section~11.4, especially Theorem~11.8]{Grigoryan2009} for the
corresponding criterion in Grigor'yan's monograph.

In particular, a Ricci lower bound implies stochastic completeness. Indeed,
assume that
\[
\operatorname{Ric}_{\mathcal M}\ge -(d-1)\kappa^2 g
\]
for some \(\kappa\ge0\). Bishop--Gromov comparison applies to any complete
\(d\)-dimensional Riemannian manifold satisfying this lower Ricci bound. It
compares the volume growth of geodesic balls in \(\mathcal M\) with the volume
growth of balls in the simply connected \(d\)-dimensional model space of
constant sectional curvature \(-\kappa^2\). Consequently,
\[
V_{\mathcal M}(x_0,r)
\le
V_{\mathbb H^d_{-\kappa^2}}(r),
\]
where the right-hand side denotes the ball volume in the corresponding
hyperbolic model space, with the Euclidean model obtained when \(\kappa=0\).
Since this model volume grows at most exponentially, there exist constants
\(A,C>0\) such that, for all sufficiently large \(r\),
\[
V_{\mathcal M}(x_0,r)\le A e^{Cr}.
\]
Consequently,
\[
\log V_{\mathcal M}(x_0,r)\le \log A+Cr,
\]
and therefore
\[
\int^\infty \frac{r\,dr}{\log V_{\mathcal M}(x_0,r)}
\ge
\int^\infty \frac{r\,dr}{\log A+Cr}
=
\infty.
\]
The volume-growth criterion then gives stochastic completeness. This shows that every complete Riemannian manifold
with Ricci curvature bounded from below is stochastically complete.

This criterion covers the geometric settings considered in the main text.

\emph{Compact manifolds.}
Every compact Riemannian manifold is geodesically complete by Hopf--Rinow. Its
Ricci curvature is bounded from below: the Ricci tensor is a smooth symmetric
\(2\)-tensor field \cite[Lemma~7.6]{LeeRiemannianManifolds}, and the continuous
function \((x,v)\mapsto\operatorname{Ric}_x(v,v)\) attains a finite minimum on
the compact unit tangent bundle \(S\mathcal M\). Hence compact Riemannian
manifolds are stochastically complete by the Ricci-lower-bound criterion.

\emph{Euclidean spaces.}
For \(\mathbb R^d\) with its standard metric, geodesic completeness is
immediate and
\[
\operatorname{Ric}=0.
\]
Thus \(\mathbb R^d\) has Ricci curvature bounded from below and is
stochastically complete.

\emph{Hyperbolic spaces.}
The \(d\)-dimensional hyperbolic space of constant sectional curvature
\(-\kappa^2\) is geodesically complete and satisfies
\[
\operatorname{Ric}=-(d-1)\kappa^2 g.
\]
Hence its Ricci curvature is bounded from below, and it is stochastically
complete.

\emph{Products and \(\mathrm{SE}(3)\).}
Products of covered examples are again covered. Indeed, geodesic completeness
is preserved under products, and the Ricci tensor of a product metric is the
product Ricci tensor. Therefore, if each factor is complete with Ricci curvature
bounded from below, then so is the product. In particular,
\[
\mathrm{SE}(3)\simeq \mathrm{SO}(3)\times\mathbb R^3,
\]
equipped with the product metric used in our experiments, is stochastically
complete: \(\mathrm{SO}(3)\) is compact, while \(\mathbb R^3\) has
\(\operatorname{Ric}=0\).

\emph{SPD cone.}
The SPD cone endowed with the affine-invariant Riemannian metric used in our
experiments is covered as well. Namely, on \(\mathrm{SPD}(n)\) we use
\[
g_X(U,V)
=
\operatorname{tr}(X^{-1}UX^{-1}V),
\]
whose induced geodesic distance is
\[
d_{\mathrm{AIRM}}(X,Y)
=
\left\|\log\left(X^{-1/2}YX^{-1/2}\right)\right\|_F .
\]
This is the standard affine-invariant metric, corresponding to the
\((\alpha,\beta)=(1,0)\) member of the affine-invariant family
\(g^{A(\alpha,\beta)}\) defined in
\cite[Definition~3.3]{ThanwerdasPennec2023}. For this family,
\cite[Proposition~3.1]{ThanwerdasPennec2023} shows that
\((\mathrm{SPD}(n),g^{A(\alpha,\beta)})\) is a Riemannian symmetric space, and
therefore geodesically complete. Moreover, the curvature formulas in
\cite[Table~5]{ThanwerdasPennec2023} show that, in particular for
\((\alpha,\beta)=(1,0)\), the sectional curvature is nonpositive and bounded
below. Since Ricci curvature is the trace of sectional curvatures over an
orthonormal basis
\cite[Ch.~7, ``Ricci and Scalar Curvatures'', after Lemma~7.6]{LeeRiemannianManifolds},
a sectional-curvature lower bound \(K\ge -K_0\) implies
\[
\operatorname{Ric}\ge -(\dim \mathrm{SPD}(n)-1)K_0\,g .
\]
Thus \(\mathrm{SPD}(n)\), and in particular \(\mathrm{SPD}(3)\), equipped with
the AIRM used in our experiments, is complete with Ricci curvature bounded from
below, and is therefore stochastically complete.

%% file: sections/review-geometric-dl.tex
Our approximation results are proved on the compact support \(K_\nu\), so we only
need a compact-domain version of the usual feature-map transfer principle from
geometric deep learning. The general philosophy is that one learns on a
non-Euclidean domain by first representing points through a continuous feature map
into a Euclidean space and then composing with a Euclidean approximator. A broad
version of this principle is developed in \cite{non_euclidean_uat}; here we
record the compact case needed in the paper.

Let \(\mathcal F\subset C(\mathbb R^n,\mathbb R)\) be dense under the topology of
uniform convergence on compact subsets. Since \(K_\nu\) is compact, approximation
on \(K_\nu\) is measured simply in the uniform norm
\[
\|g\|_\infty:=\sup_{x\in K_\nu}|g(x)|.
\]
Given a continuous feature map \(\varphi:K_\nu\to\mathbb R^n\), define the
pullback class
\begin{equation}
\label{eq-pullback-class-definition}
\varphi^*\mathcal F
:=
\{f\circ \varphi:\ f\in\mathcal F\}.
\end{equation}

The key requirement is that \(\varphi\) separate points of \(K_\nu\).

\begin{assumption}[Injective feature map]
\label{ass-feature-regularity}
The feature map \(\varphi:K_\nu\to\mathbb R^n\) is continuous and injective.
\end{assumption}

The following compact-domain transfer principle is a specialization of
\cite[Theorem~3.3]{non_euclidean_uat}.

\begin{proposition}[Transfer of Euclidean approximation through an injective feature map]
\label{prop:pullback-universality}
Assume that \(\mathcal F\subset C(\mathbb R^n,\mathbb R)\) is dense under
uniform convergence on compact subsets and that
\(\varphi:K_\nu\to\mathbb R^n\) is continuous and injective. Then
\(\varphi^*\mathcal F\) is dense in \(C(K_\nu,\mathbb R)\) under the uniform norm.
\end{proposition}

For our purposes, a particularly useful class of feature maps is given by
distance-to-landmark coordinates, following Gromov's distance-geometry
viewpoint~\citep{Gromov1983}. In the compact Riemannian setting, one
can choose finitely many landmarks so that the associated distance map is
continuous and injective. This provides an intrinsic feature map satisfying
Assumption~\ref{ass-feature-regularity}.

\begin{proposition}[Finite distance-coordinate embedding, after Gromov]
\label{prop-gromov-embedding}
Assume that $\mathcal M$ is compact. Then there exist finitely many landmarks
$x_1,\dots,x_N\in \mathcal M$ such that the map
\[
\varphi:\mathcal M\to \mathbb R^N,
\qquad
\varphi(x):=\bigl(d(x,x_1),\dots,d(x,x_N)\bigr),
\]
is continuous and injective.
In particular, $\varphi|_{K_\nu}$ satisfies
Assumption~\ref{ass-feature-regularity}.
\end{proposition}

In the Cartan--Hadamard case there is an even simpler canonical choice of feature
map. For any base point \(z\in\mathcal M\), the Cartan--Hadamard theorem implies
that the exponential map
\[
\exp_z:T_z\mathcal M \to \mathcal M
\]
is a global diffeomorphism. Hence its inverse
\[
\mathrm{Log}_z:\mathcal M\to T_z\mathcal M\simeq \mathbb R^p
\]
is globally defined and continuous, and therefore
\(\varphi:=\mathrm{Log}_z|_{K_\nu}\) automatically satisfies
Assumption~\ref{ass-feature-regularity}. Thus, on Cartan--Hadamard manifolds,
Euclidean universal approximation transfers directly through this global chart.

%% file: sections/implementation-details.tex
Unless otherwise specified, all experiments use the intrinsic quadratic
transport cost
\[
c(x,y)=\frac12 d(x,y)^2,
\]
where \(d\) is the Riemannian geodesic distance on the corresponding
manifold. Entropic RNOT parameterizes the target-side Schr\"odinger
potential with per-batch gauge centering and a two-hidden-layer MLP of width
\(256\), SiLU activations, Kaiming-normal hidden-layer initialization, and a
small-initialized scalar output layer. Manifold inputs are embedded either
through landmark/Gromov-distance features, used for
\(\mathbb S^2\), \(\mathrm{SO}(3)\), and \(\mathrm{SE}(3)\), or through
Riemannian logarithmic coordinates, used for \(\mathrm{SPD}(3)\) and
\(\mathbb H^2\). Layer normalization is applied after landmark-based
embeddings and disabled after logarithmic embeddings.

The empirical semidual objective is optimized by minibatch stochastic
gradient ascent using Adam with learning rate \(10^{-3}\), batch size \(256\),
and cosine learning-rate decay. Unless otherwise stated, the entropic
regularization is set by the median-cost heuristic
\[
\varepsilon = 0.05 \times \operatorname{median}(C),
\]
with the same value matched across Entropic RNOT and Sinkhorn-based methods. When a
deterministic transport summary is required, we extract it from the entropic
conditional distribution by heat-smoothed mode finding; the heat time is
specified separately for each experiment. All experiments are implemented in
JAX~0.6.2 with \texttt{float32} precision enabled and run on a single NVIDIA
RTX A6000 GPU with \(36\)\,GB of memory. Sinkhorn references and baselines
use OTT-JAX~\citep{Cuturi2022-ms}. Scaling results are averaged over three
random seeds, with timing measured after explicit JIT warm-up.

\subsection{Barycentric and heat-smoothed map extraction}
\label{app:barycentric-heat-extraction}

Given a trained Entropic RNOT target-side Schr\"odinger potential \(g_\theta\), a source
point \(x\), and a discrete target support \(\{y_j\}_{j=1}^M\), we first form
the entropic conditional weights
\[
w_j(x)
=
\frac{
\beta_j
\exp\!\left((g_\theta(y_j)-c(x,y_j))/\varepsilon\right)
}{
\sum_{k=1}^M
\beta_k
\exp\!\left((g_\theta(y_k)-c(x,y_k))/\varepsilon\right)
},
\]
where \(\beta_j\) denotes the target support mass. For uniform empirical
targets, \(\beta_j=1/M\), and this factor cancels. These weights define the
discrete conditional law
\[
\pi_\theta(\cdot\mid x)
=
\sum_{j=1}^M w_j(x)\,\delta_{y_j}.
\]

We consider two intrinsic deterministic summaries of this conditional law.
The first is the Riemannian barycentric projection, defined as the Fr\'echet
mean
\[
T_{\mathrm{bar}}(x)
\in
\operatorname*{argmin}_{z\in\mathcal M}
\frac12 \sum_{j=1}^M w_j(x)\, d(z,y_j)^2 .
\]
We approximate this minimizer by Karcher iterations. Starting from the target
atom with largest conditional weight,
\(z_0=y_{\arg\max_j w_j(x)}\), we repeat
\[
v_\ell
=
\sum_{j=1}^M w_j(x)\,\mathrm{Log}_{z_\ell}(y_j),
\qquad
z_{\ell+1}
=
\mathrm{Exp}_{z_\ell}\!\left(\eta_{\mathrm{bar}} v_\ell\right),
\]
followed by projection back to the manifold representation when required.
In our implementation we use \(32\) iterations with
\(\eta_{\mathrm{bar}}=0.5\). On Cartan--Hadamard manifolds, where squared
distance is globally convex along geodesics, this is gradient descent on the
Fr\'echet objective. 

The second summary is the heat-smoothed mode used in our main evaluations.
Ideally, for each conditional law one may smooth the weighted target atoms
by the manifold heat kernel,
\[
q_t(z)
=
\sum_{j=1}^M w_j(x)\,p_t(y_j,z),
\]
where \(p_t\) is the heat kernel on \(\mathcal M\). Since closed-form heat
kernels are not available uniformly across the manifolds considered here, we
use a Varadhan-type geodesic heat-kernel surrogate based on the leading
short-time logarithmic asymptotic
\[
\log p_t(y_j,z)
\approx
-\frac{d(y_j,z)^2}{4t}
=
-\frac{c(y_j,z)}{2t}.
\]
This yields the intrinsic smoothed objective
\[
\ell_t(z)
=
\log \sum_{j=1}^M
w_j(x)
\exp\!\left(-\frac{c(z,y_j)}{2t}\right),
\]
which corresponds to smoothing by a geodesic Gaussian kernel. This surrogate
keeps the smoothing intrinsic while avoiding manifold-specific heat-kernel
normalization factors, curvature corrections, and volume-distortion terms.

We maximize \(\ell_t\) by Riemannian gradient ascent. At the current iterate
\(z_\ell\), define the heat-reweighted responsibilities
\[
\alpha_j(z_\ell)
=
\frac{
w_j(x)\exp\!\left(-c(z_\ell,y_j)/(2t)\right)
}{
\sum_{k=1}^M
w_k(x)\exp\!\left(-c(z_\ell,y_k)/(2t)\right)
}.
\]
The ascent direction is the corresponding weighted log-map average,
\[
v_\ell
=
\sum_{j=1}^M \alpha_j(z_\ell)\,\mathrm{Log}_{z_\ell}(y_j),
\]
and we update
\[
z_{\ell+1}
=
\mathrm{Exp}_{z_\ell}\!\left(\eta_{\mathrm{heat}} v_\ell\right),
\]
again projecting back to the manifold representation when necessary. The
constant factor \(1/(2t)\) in the exact gradient of the surrogate objective
is absorbed into the step size. We use \(32\) ascent steps with
\(\eta_{\mathrm{heat}}=0.5\). To reduce sensitivity to initialization, the
mode finder is run from multiple intrinsic initializers built from the
conditional weights, including the heaviest target atom, and the candidate
with largest final value of \(\ell_t\) is selected.

For a discrete reference or baseline plan \(P\in\mathbb R_+^{N\times M}\), the
same extraction procedure is applied row-wise after normalizing each row:
\[
w_j^{(i)}
=
\frac{P_{ij}}{\sum_{k=1}^M P_{ik}}.
\]
Thus learned plans, Sinkhorn reference plans, and baseline plans are all
converted into deterministic summaries using the same intrinsic procedure.
For memory efficiency, the row-wise heat-smoothed extraction is evaluated in
chunks.
\subsection{Intrinsic-geometry benchmarks}
\label{app:impl-q1}

This subsection provides the implementation details for the synthetic
intrinsic-geometry benchmarks from
Section~\ref{sec:exp-synth-geometry}. We describe the manifold models, the
construction of source and target distributions, the numerical reference
solutions, the baselines, the neural semidual model, and the evaluation
protocol.

\paragraph{Manifold models and costs.}
We consider six manifold configurations spanning five manifold families.

\emph{Sphere $\mathbb S^2$.}
Points are represented as unit vectors in $\mathbb R^3$, with geodesic
distance
$$
d_{\mathbb S^2}(x,y)=\arccos(\langle x,y\rangle),
$$
computed via the numerically stable $\mathrm{atan2}$ formulation.

\emph{Rotation group $\mathrm{SO}(3)$.}
Rotations are represented as unit quaternions $q\in\mathbb R^4$ with the
antipodal identification $q\sim -q$ handled by sign normalization ($q_0\ge 0$).
The intrinsic distance is
$$
d_{\mathrm{SO}(3)}(q_1,q_2)
= 2\arccos\bigl(\min(|\langle q_1,q_2\rangle|,1))\bigr),
$$
computed via the same numerically stable $\mathrm{atan2}$ formulation as for $\mathbb{S}^2$, with antipodal identification handled by sign flipping when $\langle q_1,q_2\rangle<0$.

\emph{Rigid motions $\mathrm{SE}(3)$.}
Elements are represented as $(q,t)\in\mathbb R^7$ where $q$ is a unit
quaternion and $t\in\mathbb R^3$ is a translation. We use the weighted
product metric
$$
d_{\mathrm{SE}(3)}^2\bigl((q_1,t_1),(q_2,t_2)\bigr)
= \alpha^2\, d_{\mathrm{SO}(3)}(q_1,q_2)^2 + \|t_1 - t_2\|^2,
$$
with coupling weight $\alpha=2.0$ in the synthetic benchmark.

\emph{SPD manifold $\mathrm{SPD}(3)$, affine-invariant.}
We equip $\mathrm{SPD}(3)$ with the affine-invariant Riemannian metric
(AIRM), for which
$$
d_{\mathrm{AIRM}}(X,Y)
= \bigl\| \log\bigl(X^{-1/2}YX^{-1/2}\bigr) \bigr\|_F .
$$

\emph{SPD manifold $\mathrm{SPD}(3)$, log-Euclidean.}
We additionally evaluate on $\mathrm{SPD}(3)$ equipped with the log-Euclidean
metric
$$
d_{\mathrm{LE}}(X,Y) = \| \log X - \log Y \|_F,
$$
using the same underlying SPD support as the AIRM row. This produces a
distinct reference plan and tests the method's ability to respect a
different Riemannian structure on the same ambient space.

\emph{Hyperbolic plane $\mathbb H^2$.}
We use the Lorentz (hyperboloid) model
$$
\mathbb H^2 = \{x\in\mathbb R^3 : -x_0^2+x_1^2+x_2^2 = -1,\; x_0>0\},
$$
with geodesic distance
$$
d_{\mathbb H^2}(x,y)
= \operatorname{arcosh}\bigl(-\langle x,y\rangle_{\mathcal M}\bigr),
$$
where $\langle\cdot,\cdot\rangle_{\mathcal M}$ is the Minkowski inner product.

\paragraph{Synthetic source and target distributions.}
For each manifold, source and target distributions are constructed as
wrapped normal distributions: samples are drawn from a Gaussian in the
tangent space at a prescribed center $p\in M$ and mapped to the manifold
via the Riemannian exponential map,
$$
v \sim \mathcal N(0,\sigma^2 I_{T_pM}),
\qquad
x = \exp_p(v).
$$
The target distribution is obtained by shifting the center to a
geodesically distant point, chosen to stress curvature effects and
expose the failure of Euclidean approximations.

The specific benchmark instances are:

\begin{itemize}
\item \emph{$\mathbb S^2$.}
Source centered at the north pole $(0,0,1)$ with tangent scale $\sigma=0.7$.
Target centered at $(-0.5,0,-0.866)$, approximately $150^\circ$ from the
north pole, with $\sigma=0.7$. Both distributions wrap substantially around
the sphere, making tangent-space linearization inaccurate.

\item \emph{$\mathrm{SO}(3)$.}
Source centered at the identity quaternion. Target obtained by applying the
exponential map to a tangent vector of norm $2.5$\,rad ($\approx 143^\circ$
rotation), with tangent scale $\sigma=0.8$. This places the target near the
antipodal region where curvature distortion is maximal.

\item \emph{$\mathrm{SPD}(3)$.}
Source centered at $\mathrm{diag}(4,1,0.25)$ (condition number~16) with
$\sigma=0.5$. Target centered at a $45^\circ$ rotation of
$\mathrm{diag}(0.25,1,4)$ in the 1--3 eigenplane, yielding center
$\begin{psmallmatrix}2.125 & 0 & -1.875\\ 0 & 1 & 0\\ -1.875 & 0 & 2.125\end{psmallmatrix}$,
with $\sigma=0.5$. The non-commuting eigenbases ensure that AIRM and
log-Euclidean geodesics genuinely diverge.

\item \emph{$\mathrm{SPD}(3)(\mathrm{LE})$.}
Same source and target support as the AIRM row, but with the log-Euclidean
cost replacing the affine-invariant cost. The discrete Sinkhorn reference
is recomputed under the LE metric.

\item \emph{$\mathrm{SE}(3)$.}
Source uniform on $\mathrm{SO}(3)\times[-4,4]^3$ (Haar rotation, uniform
translation). Target a $60^\circ$ rotation about the $z$-axis paired with
translation $(1.0,0.5,-0.5)$, with $\sigma_{\mathrm{rot}}=0.3$ and
$\sigma_{\mathrm{trans}}=0.5$ truncated to $[-4,4]^3$. The weighted product
metric ($\alpha=2$) costs rotation $4\times$ translation, so the geodesic
cost is dominated by the $\mathrm{SO}(3)$ component where tangent-space
linearization is most distortive.

\item \emph{$\mathbb H^2$.}
Source centered at the origin of the hyperboloid. Target at
$\exp_o(2.0\cdot e_1)$, a geodesic distance of $\approx 2.0$ from the
origin, with $\sigma=0.5$.
\end{itemize}

\paragraph{Discrete reference entropic plan.}
For each benchmark, we compute a discrete reference solution on a shared
evaluation support of $N=200$ source and $M=200$ target samples drawn from
the respective distributions. The pairwise geodesic cost matrix
\(C_{ij}=\tfrac12 d(x_i,y_j)^2\) is formed using the intrinsic distance of
each manifold, and the entropic OT problem is solved with the Sinkhorn
algorithm in the log domain via OTT-JAX~\citep{Cuturi2022-ms}, using \(200\)
iterations. The regularization parameter \(\varepsilon\) is set by the
median-cost heuristic described above, with the median computed over
\(256\times 256\) sampled cost pairs.

When a transport map is needed from the reference plan, we extract it using
the common heat-smoothed mode-finding procedure with
\(t_{\mathrm{heat}}=100\varepsilon\).

\paragraph{Baselines.}

\emph{Ambient Euclidean.}
Points are treated as vectors in their ambient representation
($\mathbb R^3$ for $\mathbb S^2$ and $\mathbb H^2$; $\mathbb R^4$ for
$\mathrm{SO}(3)$; $\mathbb R^7$ for $\mathrm{SE}(3)$;
$\mathbb R^{3\times 3}$ vectorized for $\mathrm{SPD}(3)$). Entropic OT is
solved with squared Euclidean cost $c(x,y)=\tfrac12\|x-y\|^2$ using the
same Sinkhorn solver and matched $\varepsilon$.

\emph{Tangent-space.}
A reference point is computed as the intrinsic Fr\'echet mean of the source
samples (50 gradient steps, step size 0.5). All source and target points
are mapped to the tangent space at this reference via the Riemannian
logarithm, Euclidean entropic OT is solved with squared Frobenius cost in
that tangent space, and the transport summary is mapped back via the
exponential map. This baseline is limited to a single chart.

Both baselines use 200 Sinkhorn iterations with the same $\varepsilon$ as
Entropic RNOT and the discrete reference.

\paragraph{Entropic RNOT model.}
Entropic RNOT uses the common gauge-centered target-side
Schr\"odinger-potential parameterization described above. For the synthetic
benchmarks, the manifold-specific input representation is chosen as follows:
\begin{itemize}
\item \textbf{Landmark/Gromov-distance embeddings.}
For compact manifolds \((\mathbb S^2,\mathrm{SO}(3))\) and for the product
manifold \(\mathrm{SE}(3)\), we represent each input by its geodesic
distances to \(256\) landmarks. The landmarks are selected by farthest-point
sampling from a pool of \(4{,}096\) candidates drawn equally from the source
and target distributions.

\item \textbf{Logarithmic embeddings.}
For non-compact manifolds \((\mathrm{SPD}(3)\) with both AIRM and
log-Euclidean metrics, and \(\mathbb H^2)\), we use Riemannian logarithmic
coordinates at the manifold origin.
\end{itemize}
\paragraph{Optimization.}
For the synthetic benchmarks, training uses \(3{,}000\) stochastic-gradient
iterations. The regularization parameter, optimizer, batch size, learning-rate
schedule, and numerical precision follow the common setup above.

\paragraph{Transport map extraction.}
Given a trained Entropic RNOT potential $g_\theta$, we evaluate the Gibbs conditional
$$
\pi_\theta(dy\mid x)
\propto
\exp\!\left(\frac{g_\theta(y)-c(x,y)}{\varepsilon}\right) \nu(dy)
$$
through normalized weights on the shared target support. The transport
map is extracted via \emph{heat-smoothed mode finding}: for each source
point $x_i$, we find the mode of the heat-kernel--smoothed conditional
density.
$$
q_t(z) = \sum_j w_j\, p_t(y_j, z),
$$
where $w_j$ are the conditional weights and $p_t$ is the manifold heat
kernel at time $t_{\mathrm{heat}}=100\varepsilon$. The mode is found by
multi-start Riemannian gradient ascent ($32$ steps, step size $0.5$), with candidates initialized from the target atom with largest conditional
weight.

\paragraph{Metric computation.}
All metrics are evaluated on the shared \(200\times 200\) support. Let
\(\widehat\pi\) denote the learned or baseline plan and let \(\pi^\star\)
denote the discrete manifold Sinkhorn reference plan on that support,
computed using the intrinsic quadratic cost. To compare deterministic
transport summaries, we apply the same heat-smoothed mode extractor to every
plan. Namely, for each source point \(x_i\), we form the conditional target
weights
\[
\widehat w_j^{(i)}
=
\frac{\widehat\pi_{ij}}{\sum_k \widehat\pi_{ik}},
\qquad
w_j^{\star,(i)}
=
\frac{\pi^\star_{ij}}{\sum_k \pi^\star_{ik}},
\]
and define \(\widehat T(x_i)\) and \(T^\star(x_i)\) as the corresponding
maximizers of the heat-smoothed target density.  Thus $T^\star$ is the heat-smoothed mode of the discrete Sinkhorn reference plan, rather than an externally provided Monge map.

\emph{Plan KL.}
$\mathrm{KL}(\widehat\pi \,\|\, \pi^\star)
= \sum_{i,j} \widehat\pi_{ij} \log(\widehat\pi_{ij}/\pi^\star_{ij})$,
with a numerical floor of $10^{-30}$.

\emph{Conditional $W_1$.}
$$
\mathrm{cW}_1(\widehat\pi,\pi^\star)
= \frac{1}{N}\sum_{i=1}^N
W_1\!\bigl(\widehat\pi(\cdot\mid x_i), \pi^\star(\cdot\mid x_i)\bigr),
$$
where the inner $W_1$ is computed on the target manifold using the
intrinsic distance. This conditional metric measures how accurately
each method recovers the target conditional transport structure source
point by source point, and does not allow errors to be hidden by
rearranging mass across nearby source locations.

\emph{Map $L^2$.}
$\sqrt{(1/N)\sum_i d\bigl(\widehat T(x_i), T^\star(x_i)\bigr)^{2}}$, the RMS
geodesic error of the heat-smoothed transport summaries.

\emph{Endpoint error.}
$(1/N)\sum_i d\bigl(\widehat T(x_i), T^\star(x_i)\bigr)$, the mean geodesic
error of the transport summaries.

\paragraph{Hardware and software.}
The software stack, numerical precision, GPU, and Sinkhorn implementation are
as described in the common implementation details above.

\paragraph{Default hyperparameters.}
Unless otherwise specified, the default choices across all synthetic
benchmarks are summarized below.

\begin{center}
\small
\begin{tabular}{ll}
\toprule
Network depth & 2 \\
Hidden width & 256 \\
Batch size & 256 \\
Learning rate & $10^{-3}$ \\
Regularization $\varepsilon$ & $0.05 \times \mathrm{median}(C)$ \\
Training iterations & 3{,}000 \\
Evaluation support size & 200 \\
\bottomrule
\end{tabular}
\end{center}

\subsection{Scalability and GPU efficiency}
\label{app:impl-q2}

To examine scaling behavior, we use the same six manifold configurations and
neural architecture as Appendix~\ref{app:impl-q1}, varying the support size
\(N_x=N_y=N\in\{128,256,\dots,32{,}768\}\) over powers of two. For each
\(N\), fresh source and target samples are drawn, and timing is averaged over
three seeds. We compare Entropic RNOT against discrete manifold
Sinkhorn and Euclidean Sinkhorn. All methods use the matched median-cost regularization described in the
common implementation details.

Timing excludes
one-time JIT compilation via explicit warm-up passes. Peak memory is
recorded via JAX's \texttt{peak\_bytes\_in\_use} statistic; runs exceeding
the device budget are marked as out-of-memory.

\paragraph{Measurement protocol.}
All timings are preceded by an explicit warm-up pass to exclude JAX JIT
compilation from the timed block, and wrapped in
\texttt{block\_until\_ready} to force synchronization. Because JAX's
\texttt{peak\_bytes\_in\_use} counter is monotonic per-process, we launch
one subprocess per $(\text{manifold}, \text{method}, N)$ cell via an
orchestrator script; each subprocess runs all $3$ seeds, exits, and writes
its partial result to disk. Within a subprocess, throughput is measured
only on the last seed so that the $\mathcal{O}(N \cdot B_y)$ transient cost
matrices allocated during inference do not contaminate the training-peak
snapshots of earlier seeds. Training-time peak memory reported in the
figures is the mean over $3$ uncontaminated seeds.

\paragraph{Throughput metrics.}
\emph{Potential throughput} times the batched forward pass
$x \mapsto g_\theta(x)$ on $N$ source samples. \emph{Transport throughput}
additionally includes forming the neural entropic plan against a target
batch of $B_y = 1{,}024$ samples and running heat-smoothed mode-finding
at $t_{\mathrm{heat}} = 100\varepsilon$. \emph{Chunked transport
throughput} processes source points in Python-level chunks of $256$,
yielding $\mathcal{O}(1)$-in-$N$ inference memory
($\mathcal{O}(256 \cdot K \cdot D)$ peak by construction) at a modest throughput
cost. All throughput figures are averaged over $10$ timed calls following
$3$ warm-up calls.

\paragraph{Sinkhorn baselines.}
Both Sinkhorn baselines form the full $N \times N$ cost matrix and run
$200$ log-domain iterations via OTT-JAX~\citep{Cuturi2022-ms}; peak memory
is recorded after \texttt{block\_until\_ready}, and runs that raise
out-of-memory errors are marked as infeasible at that $N$. The Euclidean
variant differs from the manifold variant only in using squared Euclidean
cost in ambient coordinates, which isolates the contribution of intrinsic
distance evaluation to the manifold Sinkhorn runtime.

\subsection{Real-World Pose Refinement on \texorpdfstring{\(\mathrm{SE}(3)\)}{SE(3)} for Protein--Ligand Docking}
\label{app:app-real-world-experiment}

\paragraph{Motivation.}
A central computational problem in structure-based drug design is to refine candidate protein--ligand binding poses generated by docking \cite{Trott2010-qp}. Given a receptor pocket and a ligand, docking software produces a ranked list of candidate poses, each specifying a position and orientation of the ligand within the pocket. These candidates are often noisy: high-ranked poses are not always correct, and low-ranked poses are not always wrong. Experimentally determining the true binding geometry requires X-ray crystallography or cryo-EM, which are expensive and slow.

In this work, we study a more specific problem than full de novo pose prediction: crystal-free refinement of docking pose ensembles using only information available at docking time. Concretely, we treat the candidate poses produced by a docking engine as empirical distributions on $\mathrm{SE}(3)$, and learn an Entropic RNOT refinement map that moves geometric outliers toward the docking engine's own top-ranked binding basin. Accordingly, our method should be interpreted as a docking-pose refinement or denoising procedure, not as a direct predictor of the crystal pose. The crystallographic pose $X_k^\star$ is used only for held-out evaluation, never for training or inference.

\paragraph{Dataset.}
We use CrossDocked2020~\citep{Francoeur2020-mc}, a benchmark of protein--ligand complexes organized by pocket-similarity clusters. Let $\mathcal{K}$ denote the set of complexes and let $\mathcal{C}$ denote the set of pocket-similarity clusters, where each complex $k \in \mathcal{K}$ is assigned to a cluster $c(k) \in \mathcal{C}$. Two complexes in the same cluster share structurally related binding-site geometry. The dataset contains $|\mathcal{K}| = 3{,}765$ complexes across $|\mathcal{C}| = 1{,}302$ clusters. For each complex $k$, the available inputs are:
\begin{itemize}
    \item receptor pocket heavy atoms $\mathcal{P}_k = \{p_{k,b}\}_{b=1}^{B_k} \subset \mathbb{R}^3$, where $B_k$ is the number of non-hydrogen atoms in the binding site, the local protein environment that the ligand binds to;
    \item a canonical ligand conformer $L_k = \{\ell_{k,a}\}_{a=1}^{A_k} \subset \mathbb{R}^3$, where $A_k$ is the number of ligand heavy atoms. This is a single known 3D geometry of the drug-like small molecule, used as the reference shape for rigid-body alignment;
    \item a crystallographic ligand pose $X_k^\star = \{x_{k,a}^\star\}_{a=1}^{A_k} \subset \mathbb{R}^3$, the experimentally determined binding geometry. We reserve the crystallographic ligand pose \emph{for held-out evaluation only}, and it is s not used in training or pretraining.
\end{itemize}

\paragraph{Receptor-defined local frame.}
For each complex $k$, we define a deterministic orthonormal frame from receptor pocket heavy atoms $\mathcal{P}_k = \{p_{k,b}\}_{b=1}^{B_k} \subset \mathbb{R}^3$ alone. The origin is the pocket centroid $c_k = B_k^{-1} \sum_b p_{k,b}$. We select the three residue centroids $r_{k,1}, r_{k,2}, r_{k,3}$ nearest to $c_k$ and apply Gram--Schmidt orthonormalization to obtain $Q_k = [e_{k,1}\ e_{k,2}\ e_{k,3}] \in \mathrm{SO}(3)$, with a sign convention that makes the frame deterministic. All ligand and pose coordinates are transformed into this frame via $\bar{x} = Q_k^\top(x - c_k)$, so that the subsequent $\mathrm{SE}(3)$ elements are defined relative to receptor geometry only. No ligand or crystal information is used.
\paragraph{Docking and rigid-body alignment.}
For each complex $k$, we generate $M = 40$ candidate docked poses using GNINA~\citep{McNutt2025-np} with the Vinardo~\citep{Quiroga2016-je} scoring function. GNINA takes as input the receptor pocket structure $\mathcal{P}_k$ and the canonical ligand conformer $L_k$. The docking search box is centered at the pocket centroid $c_k$ with side lengths equal to the coordinate-wise pocket extent plus a margin of $\delta_{\mathrm{box}} = 4$~\AA, defining a large search region from receptor geometry alone. We disable GNINA's CNN-based rescoring (\texttt{-{}-cnn\_scoring none}) to avoid data leakage from models trained on PDBbind crystal poses. This yields posed ligand coordinates $\{X_{k,m}\}_{m=1}^{M}$ with $X_{k,m} = \{x_{k,m,a}\}_{a=1}^{A_k} \subset \mathbb{R}^3$ and associated docking scores $\{s_{k,m}\}_{m=1}^{M}$. Note that both the canonical conformer $L_k$ (the input geometry given to the docking tool) and the docked poses $X_{k,m}$ (the output of the docking tool) are available at inference time; neither requires the crystallographic pose $X_k^\star$. All coordinates are expressed in the receptor-defined local frame.

Each docked pose is converted to a rigid-body transform by Kabsch alignment to the canonical conformer:
\[
(R_{k,m}, t_{k,m}) = \operatorname*{argmin}_{R \in \mathrm{SO}(3),\, t \in \mathbb{R}^3} \frac{1}{A_k} \sum_{a=1}^{A_k} \|R\,\ell_{k,a} + t - x_{k,m,a}\|_2^2,
\]
yielding a pose $g_{k,m} = (R_{k,m}, t_{k,m}) \in \mathrm{SE}(3)$ for each candidate. This representation is exact when the docked pose differs from the canonical conformer by rigid motion only. For flexible ligands, internal torsion changes during docking introduce error that rigid motion cannot capture. We quantify this by the rigid-fit residual
\[
r_{k,m} = \left(\frac{1}{A_k} \sum_{a=1}^{A_k} \|R_{k,m}\,\ell_{k,a} + t_{k,m} - x_{k,m,a}\|_2^2\right)^{1/2}.
\]
We discard poses with $r_{k,m} > \tau_{\mathrm{rigid}} = 2.5$~\AA, retaining 155 complexes and 5{,}061 poses. After the source/target split and cluster-based train/test partition, this yields 117 training complexes (3{,}567 source, 559 target) across 75 clusters and 29 test complexes (814 source, 110 target) across 18 held-out clusters.

Since our pose representation models each docked ligand as a rigid transform of a fixed conformer, the near-rigid subset provides the most faithful setting for evaluating the method under its intended geometric assumptions. We therefore use this subset as our primary benchmark. At the same time, to test robustness beyond this idealised regime, we also evaluate on the full unfiltered docking ensembles. Performance remains comparable and is in several metrics even stronger without rigid filtering, suggesting that the method is reasonably robust to violations of the rigid-pose assumption. See Table~\ref{tab:se3-refinement_norigid}. The consistency between the restricted and unfiltered evaluations also suggests that the conclusions drawn from the near-rigid benchmark are not merely an artifact of reduced sample size.

\paragraph{Source and target distributions.}
For each complex $k$, let $g_{k,1}^*$ denote the top-ranked pose by Vinardo score. We partition the retained poses into target and source sets by geometric proximity to the docking tool's highest-confidence prediction:
\[
I_k^{\mathrm{tgt}} = \{m : d_{\mathrm{SE}(3)}(g_{k,m}, g_{k,1}^*) \leq \delta_{\mathrm{prox}}\}, \qquad I_k^{\mathrm{src}} = \{m : d_{\mathrm{SE}(3)}(g_{k,m}, g_{k,1}^*) > \delta_{\mathrm{prox}}\},
\]

with $\delta_{\mathrm{prox}} = 5.0$~\AA, the standard binding-mode tolerance used in docking evaluation~\citep{McNutt2025-np, Corso2022-si}. Target poses are those that lie within the prescribed $\mathrm{SE}(3)$
neighborhood of the top-ranked prediction; source poses are geometric outliers. Complexes that yield empty source or target sets are discarded. This construction uses only the docking scores and the $\mathrm{SE}(3)$ metric; no crystallographic information is consulted. This construction makes the learning problem entirely crystal-free: the transport targets are defined solely from docking-generated poses and their geometry, without reference to the experimental structure. At the same time, the learned refinement is toward the docking engine's self-consistent top-ranked basin rather than directly toward the crystallographic pose.

The source and target empirical measures are
\[
\mu_k = \frac{1}{|I_k^{\mathrm{src}}|} \sum_{m \in I_k^{\mathrm{src}}} \delta_{g_{k,m}}, \qquad \nu_k = \frac{1}{|I_k^{\mathrm{tgt}}|} \sum_{m \in I_k^{\mathrm{tgt}}} \delta_{g_{k,m}}.
\]

\paragraph{Training and testing split.}
We split by pocket-similarity cluster to prevent structural leakage: if two complexes share a similar binding site, they are assigned to the same partition. With a 80/20 split, this yields 75 training clusters (117 complexes, 3{,}567 source and 559 target poses) and 18 held-out test clusters (29 complexes, 814 source and 110 target poses). No cluster appears in both partitions.

\paragraph{Empirical $\mathrm{SE}(3)$ metric.}
We equip $\mathrm{SE}(3)$ with the product metric $d_{\mathrm{SE}(3)}(g,h)^2 = \alpha^2 d_{\mathrm{SO}(3)}(R_g, R_h)^2 + \|t_g - t_h\|_2^2$. The weight $\alpha$ is set to the median radius of gyration of training ligands, $\alpha = \tilde{r}_{\mathrm{gyr}} = 2.44$~\AA\ (computed from 117 training complexes), so that a rotation of $\theta$ rad contributes approximately $r_{\mathrm{gyr}} \cdot \theta$~\AA\ to the distance, making rotational and translational components commensurate with heavy-atom RMSD. All training inputs are derived from the docking pipeline described above.

\paragraph{Training.}
Source and target poses are pooled across training complexes into empirical
distributions \(\mu\) and \(\nu\) on \(\mathrm{SE}(3)\). The entropic
regularization is set to
\(\varepsilon = 0.05 \times \mathrm{median}(C) = 16.19\), where \(C\) is the
pairwise \(\mathrm{SE}(3)\) cost matrix on a calibration subsample. The Entropic RNOT semidual potential \(\psi_\theta\) uses the common landmark-distance
embedding and MLP architecture. Training follows the common optimizer setup,
but uses \(5{,}000\) steps for this experiment, taking approximately
\(14\) seconds on one GPU.

\paragraph{Evaluation.}
Only at evaluation do we use the crystallographic pose $X_k^\star$. For each test complex $k$, we compute the crystal rigid-body transform $g_k^\star = (R_k^\star, t_k^\star) \in \mathrm{SE}(3)$ by the same Kabsch alignment to the canonical conformer in the receptor-defined local frame. For each source pose $g \in \mathrm{spt}(\mu_k)$, we compute the refined pose from the learned Entropic RNOT conditional using
heat-smoothed mode finding with \(t=\varepsilon\). At $t = \varepsilon$, the heat-smoothed mode concentrates on the locally dominant cluster of the entropic conditional rather than averaging across geometrically distant modes. This is particularly important on $\mathrm{SE}(3)$, where the conditional $\pi(\cdot \mid g)$ can be multimodal: a source pose far from the target region may have non-negligible weight on multiple distinct binding modes, and the Fréchet mean of these modes (recovered at large $t$) may lie in a geometrically invalid region of pose space between them.
 
Per-complex top-1 metrics (best transported pose per complex) are reported to enable comparison with other methods. All metrics are averaged over the 29 held-out test complexes (814 source poses, 18 clusters) and and bootstrapping with $100$ resamples is used for $95\%$ confidence intervals.

As a baseline, we also run GNINA's built-in local energy minimization (\texttt{-{}-minimize}) on the same test source poses using the Vinardo scoring function, which refines each pose by optimization on the docking energy surface. This baseline uses the same information available at inference time (receptor, ligand, and docking scores) but operates independently on each pose without any learned transport.

As a further baseline, we run per-complex discrete Sinkhorn directly on $\mathrm{SE}(3)$: for each test complex, we form the intrinsic geodesic cost matrix, solve the entropic OT problem with matched $\varepsilon$ ($200$ log-domain iterations via OTT-JAX), and extract the transport map by the same heat-smoothed mode-finding used for Entropic RNOT. This isolates the benefit of amortized cross-complex learning: with only a handful of target poses per complex, the per-complex discrete plan is severely under-determined, and the baseline's performance quantifies how much worse one does by solving each complex in isolation rather than pooling across the training distribution. This is visualized in Figure  \ref{fig:schematic_dock}

%% file: sections/algorithms.tex
\begin{algorithm}[htbp]
\caption{Training Entropic RNOT via the neural semidual objective}
\label{alg:neural-semidual-eot}
\begin{algorithmic}[1]
\STATE \textbf{Input:} samples from \(\mu\) and \(\nu\), regularization \(\varepsilon>0\), feature map \(\varphi\), neural potential \(a_\theta\)
\FOR{each training iteration}
    \STATE Sample minibatches \(\{x_i\}_{i=1}^B\sim\mu\) and \(\{y_j\}_{j=1}^B\sim\nu\)
    \STATE Evaluate and center the target-side potential
    \[
    g_\theta(y_j)
    =
    a_\theta(y_j)-\frac1B\sum_{\ell=1}^B a_\theta(y_\ell)
    \]
    \STATE Compute the empirical soft \(c\)-transform
    \[
    f_\theta^\varepsilon(x_i)
    =
    -\varepsilon\log\!\left(
    \frac1B\sum_{j=1}^B
    \exp\!\left(
    \frac{g_\theta(y_j)-c(x_i,y_j)}{\varepsilon}
    \right)
    \right)
    \]
    \STATE Form \(\widehat{\mathcal J}_\varepsilon(\theta)\) and take a gradient ascent step
\ENDFOR
\STATE \textbf{Output:} Entropic RNOT potential \(g_\theta\)
\end{algorithmic}
\end{algorithm}

\begin{algorithm}[h!]
\caption{Intrinsic barycentric map induced by Entropic RNOT}
\label{alg:intrinsic-entropic-barycentric-map-compact}
\begin{algorithmic}[1]
\STATE \textbf{Input:} source point \(x\), target support \(\{y_j\}_{j=1}^K\sim\nu\), regularization \(\varepsilon>0\), neural potential \(\psi_\theta\)
\STATE Evaluate the target-side potential \(\psi_\theta(y_j)\)
\STATE Compute the entropic conditional weights
\[
w_j(x)
=
\frac{
\exp\!\left(\frac{\psi_\theta(y_j)-c(x,y_j)}{\varepsilon}\right)
}{
\sum_{k=1}^K
\exp\!\left(\frac{\psi_\theta(y_k)-c(x,y_k)}{\varepsilon}\right)
}
\]
\STATE Initialize \(z_0\) from the heaviest atom \(y_{j^\star}\), where
\[
j^\star\in\arg\max_{j=1,\dots,K} w_j(x)
\]
\FOR{\(\ell=0,\dots,T-1\)}
    \STATE Compute the intrinsic barycenter direction
    \[
    v_\ell
    =
    \sum_{j=1}^K w_j(x)\mathrm{Log}_{z_\ell}(y_j)
    \]
    \STATE Update
    \[
    z_{\ell+1}
    =
    \mathrm{Exp}_{z_\ell}(\eta\, v_\ell)
    \]
\ENDFOR
\STATE \textbf{Output:} intrinsic entropic barycentric projection \(T_\varepsilon(x)=z_T\)
\end{algorithmic}
\end{algorithm}

\begin{algorithm}[htbp]
\caption{Heat-smoothed transport extraction from an Entropic RNOT plan}
\label{alg:heat-smoothed-transport}
\begin{algorithmic}[1]
\STATE \textbf{Input:} transport plan \(P\in\mathbb{R}_{+}^{N\times M}\), target support \(\{y_j\}_{j=1}^M\subset\mathcal M\), diffusion time \(t>0\), manifold operations \(\mathrm{Log},\mathrm{Exp}\), initializer set \(\{z_0^{(m)}\}_{m=1}^L\)
\FOR{each source index \(i=1,\dots,N\)}
    \STATE Normalize the \(i\)-th plan row into conditional weights
    \[
    w_j^{(i)}=\frac{P_{ij}}{\sum_{k=1}^M P_{ik}},\qquad j=1,\dots,M
    \]
    \FOR{each initialization \(m=1,\dots,L\)}
        \STATE Set \(z^{(m)}_0=z_0^{(m)}\in\mathcal M\)
        \FOR{\(\ell=0,\dots,T-1\)}
            \STATE Compute the heat-kernel reweighted conditional coefficients
            \[
            \alpha_{j,\ell}^{(m)}
            =
            \frac{
            w_j^{(i)}
            \exp\!\bigl(-c(z_\ell^{(m)},y_j)/(2t)\bigr)
            }{
            \sum_{k=1}^M
            w_k^{(i)}
            \exp\!\bigl(-c(z_\ell^{(m)},y_k)/(2t)\bigr)
            }
            \]
            \STATE Form the intrinsic ascent direction
            \[
            v_\ell^{(m)}
            =
            \sum_{j=1}^M
            \alpha_{j,\ell}^{(m)}
            \mathrm{Log}_{z_\ell^{(m)}}(y_j)
            \]
            \STATE Take a Riemannian gradient ascent step
            \[
            z_{\ell+1}^{(m)}
            =
            \mathrm{Exp}_{z_\ell^{(m)}}\!\bigl(\eta\, v_\ell^{(m)}\bigr)
            \]
        \ENDFOR
        \STATE Evaluate the heat-smoothed log-density
        \[
        \log q_t\!\bigl(z_T^{(m)}\bigr)
        =
        \log\!\left(
        \sum_{j=1}^M
        w_j^{(i)}\,p_t(y_j,z_T^{(m)})
        \right)
        \approx
        \log\!\left(
        \sum_{j=1}^M
        w_j^{(i)}
        \exp\!\bigl(-c(y_j,z_T^{(m)})/(2t)\bigr)
        \right)
        \]
    \ENDFOR
    \STATE Select the best mode across initializations
    \[
    T(x_i)=z_T^{(m^\star)},
    \qquad
    m^\star\in\arg\max_{m=1,\dots,L}\log q_t\!\bigl(z_T^{(m)}\bigr)
    \]
\ENDFOR
\STATE \textbf{Output:} transported points \(\{T(x_i)\}_{i=1}^N\subset\mathcal M\)
\end{algorithmic}
\end{algorithm}

\begin{algorithm}[htbp]
\caption{Heat-smoothed Gibbs mode extraction on a manifold}
\label{alg:heat-smoothed-mode}
\begin{algorithmic}[1]
\STATE \textbf{Input:} conditional weights \(\{w_j\}_{j=1}^M\), target support \(\{y_j\}_{j=1}^M\subset\mathcal M\), diffusion time \(t>0\), step size \(\eta\), initializers \(\{z_0^{(m)}\}_{m=1}^L\)
\FOR{each initialization \(m=1,\dots,L\)}
    \STATE Set \(z_0^{(m)}\in\mathcal M\)
    \FOR{\(\ell=0,\dots,T-1\)}
        \STATE Compute heat-smoothed weights
        \[
        \alpha_{j,\ell}^{(m)}
        =
        \frac{
        w_j \exp\!\bigl(-c(z_\ell^{(m)},y_j)/(2t)\bigr)
        }{
        \sum_{k=1}^M
        w_k \exp\!\bigl(-c(z_\ell^{(m)},y_k)/(2t)\bigr)
        }
        \]
        \STATE Form the intrinsic ascent direction
        \[
        v_\ell^{(m)}
        =
        \sum_{j=1}^M \alpha_{j,\ell}^{(m)}\mathrm{Log}_{z_\ell^{(m)}}(y_j)
        \]
        \STATE Update by exponential map
        \[
        z_{\ell+1}^{(m)}
        =
        \mathrm{Exp}_{z_\ell^{(m)}}\!\bigl(\eta\, v_\ell^{(m)}\bigr)
        \]
    \ENDFOR
    \STATE Compute the final objective
    \[
    \log q_t\!\bigl(z_T^{(m)}\bigr)
    =
    \log\!\left(\sum_{j=1}^M w_j\,p_t(y_j,z_T^{(m)})\right)
    \]
\ENDFOR
\STATE Return the maximizer
\[
z^\star \in \arg\max_{m=1,\dots,L}\log q_t\!\bigl(z_T^{(m)}\bigr)
\]
\STATE \textbf{Output:} heat-smoothed mode \(z^\star\)
\end{algorithmic}
\end{algorithm}

%% file: sections/app-additional-experimental-results.tex
\subsubsection*{Additional synthetic results}

Table~\ref{tab:synthetic-manifold-benchmarks} expands Table~\ref{tab:synthetic-benchmarks} with two additional metrics — conditional $W_1$ ($\mathrm{cW}_1$) at the plan level and intrinsic RMS geodesic error ($\mathrm{Map~}L^2$) at the map level. We also include one additional manifold, $\mathrm{SPD}(3)$~(LE), which evaluates the same SPD support under the log-Euclidean cost instead of the affine-invariant Riemannian metric.

\begin{table*}[h!]
\centering
\small
\caption{Synthetic transport benchmarks across manifold geometries. Lower is
better for all metrics. Plan errors are computed relative to a discrete
manifold Sinkhorn reference on a sampled support.}
\label{tab:synthetic-manifold-benchmarks}
\begin{tabular}{llcccc}
\toprule
Manifold & Method & Plan KL $\downarrow$ & $\mathrm{cW}_1$ $\downarrow$ & Map $L^2$ $\downarrow$ & Endpoint error $\downarrow$ \\
\midrule
\multirow{3}{*}{$\mathbb{S}^2$}
& Ambient Euclidean   & 0.7037 & 0.2269 & 0.3168 & 0.2141 \\
& Tangent-space       & 0.4780 & 0.2133 & 0.3089 & 0.2090 \\
& Entropic RNOT                & \textbf{0.0461} & \textbf{0.0762} & \textbf{0.0943} & \textbf{0.0752} \\
\midrule
\multirow{3}{*}{$\mathrm{SO}(3)$}
& Ambient Euclidean   & 1.4480 & 0.6807 & 0.8684 & 0.7792 \\
& Tangent-space       & 0.1514 & 0.2012 & 0.3486 & 0.2523 \\
& Entropic RNOT                & \textbf{0.0660} & \textbf{0.1825} & \textbf{0.2514} & \textbf{0.2434} \\
\midrule
\multirow{3}{*}{$\mathrm{SPD}(3)$}
& Ambient Euclidean   & 1.8471 & 0.8733 & 2.6820 & 0.6101 \\
& Tangent-space       & 1.2300 & 0.7044 & 1.2423 & 0.4293 \\
& Entropic RNOT                & \textbf{0.0085} & \textbf{0.0657} & \textbf{0.0942} & \textbf{0.0419} \\
\midrule
\multirow{3}{*}{$\mathrm{SPD}(3)$ (LE)}
& Ambient Euclidean   & 1.6509 & 0.7671 & 2.0240 & 0.5291 \\
& Tangent-space       & 1.5816 & 0.7466 & 1.3729 & 0.4831 \\
& Entropic RNOT                & \textbf{0.0268} & \textbf{0.1063} & \textbf{0.3016} & \textbf{0.0771} \\
\midrule
\multirow{3}{*}{$\mathrm{SE}(3)$}
& Ambient Euclidean   & 1.7234 & 0.8250 & 0.2285 & 0.6432 \\
& Tangent-space       & 1.3202 & 0.6674 & 0.2117 & 0.5050 \\
& Entropic RNOT                & \textbf{0.0553} & \textbf{0.1628} & \textbf{0.0787} & \textbf{0.1200} \\
\midrule
\multirow{3}{*}{$\mathbb{H}^2$}
& Ambient Euclidean   & 0.9612 & 0.3126 & 2.1070 & 0.2448 \\
& Tangent-space       & 0.0995 & 0.1166 & 0.6823 & 0.1063 \\
& Entropic RNOT                & \textbf{0.0095} & \textbf{0.0471} & \textbf{0.1427} & \textbf{0.0416} \\
\bottomrule
\end{tabular}
\end{table*}

\subsubsection*{Additional pose refinement results}

Table~\ref{tab:se3-refinement_norigid} extends Table~\ref{tab:se3-refinement} to the full unfiltered CrossDocked2020 ensemble, dropping the $2.5$\,\AA{} rigid-fit residual cutoff used in the main table to restrict evaluation to near-rigid poses. The same per-complex top-1 metrics are reported on the broader set, which includes poses exhibiting non-trivial conformational deviation from rigid alignment.

\begin{table}[h!]
\centering
\small
\caption{Post-docking pose refinement on $\mathrm{SE}(3)$ (CrossDocked2020, 54 held-out clusters). Per-complex top-1 metrics with 95\% bootstrap CIs. Crystal pose used for evaluation only. No rigid filter used.}
\label{tab:se3-refinement_norigid}
\begin{tabular}{lcccc}
\toprule
Method & RMSD (\AA) $\downarrow$ & Median (\AA) $\downarrow$ & @2\AA\ $\uparrow$ & @5\AA\ $\uparrow$ \\
\midrule
No refinement & 9.14$_{\scriptstyle[8.07,10.25]}$ & 9.10$_{\scriptstyle[7.64,9.90]}$ & 9.3\%$_{\scriptstyle[2.7,18.5]}$ & 20.4\%$_{\scriptstyle[11.1,28.7]}$ \\
GNINA --minimize & 9.17$_{\scriptstyle[7.77,11.02]}$ & 8.92$_{\scriptstyle[6.39,9.85]}$ & 0.0\%$_{\scriptstyle[0.0,0.0]}$ & 18.8\%$_{\scriptstyle[6.2,34.4]}$ \\
Sinkhorn SE(3) & 10.08$_{\scriptstyle[7.13,13.23]}$ & 4.03$_{\scriptstyle[2.10,7.43]}$ & 29.6\%$_{\scriptstyle[21.2,40.7]}$ & 53.7\%$_{\scriptstyle[40.7,63.0]}$ \\
Entropic RNOT & \textbf{1.68$_{\scriptstyle[1.43,2.00]}$} & \textbf{1.57$_{\scriptstyle[1.33,1.82]}$} & \textbf{72.2\%$_{\scriptstyle[60.1,83.3]}$} & \textbf{98.1\%$_{\scriptstyle[93.5,100.0]}$} \\
\bottomrule
\end{tabular}
\end{table}

\subsubsection*{Entropic versus non-entropic RNOT}

Table~\ref{tab:chnot-vs-rnot-s2} compares Entropic RNOT with the
non-entropic RNOT model of~\cite{micheli2026riemannianneuraloptimaltransport},
which requires an inner solve. Under matched architecture and batch size, the
two methods reach comparable measure approximation on $\mathbb{S}^2$:
pushforward Wasserstein distances to $\nu$ are similar, while Entropic RNOT is
substantially faster and more memory efficient. Training time is shown in
Figure~\ref{fig:chnot-vs-rnot-s2}.

\begin{table}[t]
\centering
\small
\caption{Entropic RNOT versus non-entropic RNOT on $\mathbb{S}^2$,
mean $\pm$ std over 5 training seeds. Quality is the geodesic $W_p$
between $T_\#\mu$ and a held-out target sample; cost is training
wall-clock and peak GPU memory.}
\label{tab:chnot-vs-rnot-s2}
\begin{tabular}{lcc}
\toprule
Metric & Entropic RNOT (entropic) & RNOT (non-entropic) \\
\midrule
\multicolumn{3}{l}{\emph{Quality}} \\
$W_1(T_\#\mu, \nu)$ $\downarrow$ & 0.1614 $\pm$ 0.0091 & 0.1850 $\pm$ 0.0068 \\
$W_2(T_\#\mu, \nu)$ $\downarrow$ & 0.1988 $\pm$ 0.0123 & 0.2263 $\pm$ 0.0083 \\
\midrule
\multicolumn{3}{l}{\emph{Cost} ($N_{\text{steps}}$ = 3000 / 200, batch = 256)} \\
Train wall-clock (s) $\downarrow$ & 24.2 $\pm$ 0.6 & 82.3 $\pm$ 0.1 \\
Peak memory (MB) $\downarrow$ & 88 & 201 \\
\bottomrule
\end{tabular}
\end{table}

\begin{figure}[ht]
\centering
\includegraphics[width=\textwidth]{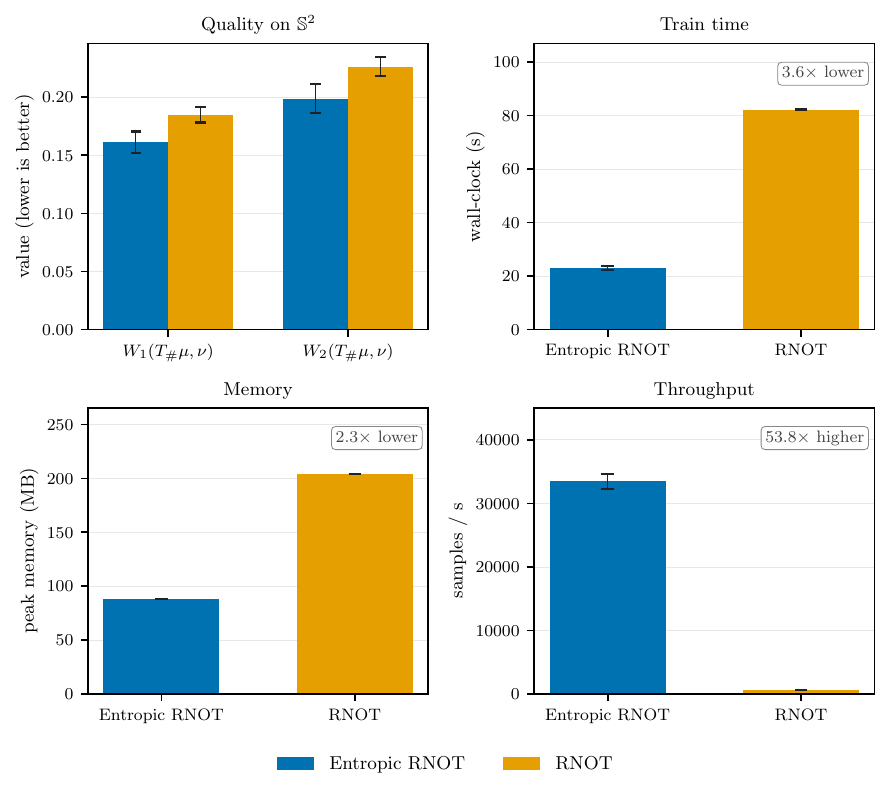}
\caption{Entropic RNOT versus non-entropic RNOT on $\mathbb{S}^2$, mean
$\pm$ std over training seeds. \emph{Leftmost}: pushforward Wasserstein
distances $W_1, W_2$ between $T_\#\mu$ and a held-out target sample of
$\nu$. \emph{Right three panels}: training wall-clock, peak GPU memory, and
training throughput under matched architecture and batch size; annotations
report the relative ratio of non-entropic RNOT versus Entropic RNOT.}

\label{fig:chnot-vs-rnot-s2}
\end{figure}

\subsubsection*{Heat-smoothed transport map and barycentric-projection comparison}

Table~\ref{tab:extractor-ablation-ch} compares two transport extractors on the Cartan--Hadamard manifolds ($\mathrm{SPD}(3)$, $\mathrm{SPD}(3)$~(LE), and $\mathbb{H}^2$): the heat-smoothed mode and the Riemannian barycentric projection, which Proposition~\ref{thm:main-fixed-eps-map} shows makes the recovery guarantee well-defined on these manifolds. We use a large heat time $t$; in this regime, the heat-smoothed mode-finding
iteration approaches the Riemannian barycenter. The two extractors agree to within $10^{-3}$ on every row. We therefore use barycentric projection for $\mathrm{SPD}(3)$, $\mathrm{SPD}(3)$~(LE), and $\mathbb{H}^2$, and the heat-smoothed mode for the compact and mixed-compact manifolds, where the Fr\'echet mean may fail to be unique.

\begin{table}[ht]
\centering
\small
\caption{Effect of transport extractors on Cartan--Hadamard manifolds}
\label{tab:extractor-ablation-ch}
\setlength{\tabcolsep}{4pt}
\begin{tabular}{lcccc}
\toprule
              & \multicolumn{2}{c}{Heat-smoothed mode} & \multicolumn{2}{c}{Riemannian barycenter} \\
\cmidrule(lr){2-3} \cmidrule(lr){4-5}
Method        & Map $L^2$ $\downarrow$ & Endpt err $\downarrow$ & Map $L^2$ $\downarrow$ & Endpt err $\downarrow$ \\
\midrule
\multicolumn{5}{l}{\emph{$\mathrm{SPD}(3)$}} \\
\quad Ambient Euclidean  & 0.6806 & 0.6101 & 0.6802 & 0.6098 \\
\quad Tangent-space      & 0.4688 & 0.4293 & 0.4676 & 0.4286 \\
\quad Entropic RNOT               & 0.0439 & 0.0419 & 0.0438 & 0.0419 \\
\midrule
\multicolumn{5}{l}{\emph{$\mathrm{SPD}(3)$ (LE)}} \\
\quad Ambient Euclidean  & 0.5988 & 0.5291 & 0.5983 & 0.5288 \\
\quad Tangent-space      & 0.5323 & 0.4831 & 0.5318 & 0.4828 \\
\quad Ours               & 0.0859 & 0.0771 & 0.0857 & 0.0770 \\
\midrule
\multicolumn{5}{l}{\emph{$\mathbb{H}^2$}} \\
\quad Ambient Euclidean  & 0.3612 & 0.2448 & 0.3606 & 0.2447 \\
\quad Tangent-space      & 0.1428 & 0.1063 & 0.1427 & 0.1062 \\
\quad Ours               & 0.0424 & 0.0416 & 0.0424 & 0.0415 \\
\bottomrule
\end{tabular}
\end{table}

%% file: sections/proofs.tex
\subsection{Proof of Theorem~\ref{thm:main-fixed-eps-plan}}
\label{app-proof-recovery-entropic-plan}

We work under the standing assumptions of Section~\ref{sec-theory}. In
particular, \(K_\mu\) and \(K_\nu\) are compact, the cost is
\(c(x,y)=\frac12 d(x,y)^2\), and the normalized entropic Schr\"odinger
potential \(g_\varepsilon^\star\) is continuous on \(K_\nu\). We use the
pullback class \(\varphi^*\mathcal F\) and the centered pullback class
\(\mathsf C_\nu(\varphi^*\mathcal F)\) as defined in
\eqref{eq:pullback-class-general} and \eqref{eq-centered-pullback-class}.

We first record that the centering operator \(\mathsf C_\nu\), defined in
\eqref{eq-centered-pullback-class}, preserves uniform density on the centered
subspace.

\begin{lemma}[Density of the centered class]
\label{lem:centered-pullback-class}
Let \(\mathcal A\subset C(K_\nu)\) be dense in \(C(K_\nu)\) with respect to
\(\|\cdot\|_{L^\infty(K_\nu)}\). 
Then:
\begin{enumerate}
\item[(i)] every \(g\in \mathsf C_\nu(\mathcal A)\) satisfies
\[
\int_{K_\nu} g\,d\nu=0;
\]
\item[(ii)] for every \(h\in C(K_\nu)\) with
\[
\int_{K_\nu} h\,d\nu=0,
\]
there exists a sequence \((h_m)_{m\in\mathbb N}\subset \mathsf C_\nu(\mathcal A)\)
such that
\[
\|h_m-h\|_{L^\infty(K_\nu)}\xrightarrow[m\to\infty]{}0.
\]
\end{enumerate}
\end{lemma}

\begin{proof}
For (i), let \(g\in \mathsf C_\nu(\mathcal A)\). Then \(g=\mathsf C_\nu a\)
for some \(a\in\mathcal A\), and hence
\[
\int_{K_\nu} g\,d\nu
=
\int_{K_\nu}
\left(a-\int_{K_\nu}a\,d\nu\right)d\nu
=
\int_{K_\nu}a\,d\nu
-
\left(\int_{K_\nu}a\,d\nu\right)\nu(K_\nu)
=0,
\]
because \(\nu(K_\nu)=1\).

For (ii), let \(h\in C(K_\nu)\) satisfy
\[
\int_{K_\nu}h\,d\nu=0.
\]
Since \(\mathcal A\) is dense in \(C(K_\nu)\), there exists a sequence
\((a_m)_{m\in\mathbb N}\subset\mathcal A\) such that
\[
\|a_m-h\|_{L^\infty(K_\nu)}\to 0.
\]
Define
\[
h_m:=\mathsf C_\nu a_m
=
a_m-\int_{K_\nu}a_m\,d\nu .
\]
Then \(h_m\in \mathsf C_\nu(\mathcal A)\). Moreover, since \(h\) is centered,
\[
\int_{K_\nu}a_m\,d\nu
=
\int_{K_\nu}(a_m-h)\,d\nu,
\]
and therefore
\[
h_m-h
=
a_m-h-\int_{K_\nu}(a_m-h)\,d\nu .
\]
Thus
\[
\begin{aligned}
\|h_m-h\|_{L^\infty(K_\nu)}
&\le
\|a_m-h\|_{L^\infty(K_\nu)}
+
\left|\int_{K_\nu}(a_m-h)\,d\nu\right|  \\
&\le
\|a_m-h\|_{L^\infty(K_\nu)}
+
\int_{K_\nu}|a_m-h|\,d\nu  \\
&\le
2\|a_m-h\|_{L^\infty(K_\nu)}.
\end{aligned}
\]
The right-hand side tends to \(0\), so
\[
\|h_m-h\|_{L^\infty(K_\nu)}\to 0.
\]
This proves (ii), and hence the lemma.
\end{proof}

We next record the elementary stability of the soft \(c\)-transform under
uniform perturbations.

\begin{lemma}[Supremum-norm stability of the soft \(c\)-transform]
\label{lem:supnorm-stability}
For all \(g,h\in C(K_\nu)\),
\[
\|\mathcal T_\nu^\varepsilon g-\mathcal T_\nu^\varepsilon h\|_{L^\infty(K_\mu)}
\le
\|g-h\|_{L^\infty(K_\nu)}.
\]
\end{lemma}

\begin{proof}
Set
\[
\delta:=\|g-h\|_{L^\infty(K_\nu)}.
\]
Then, for every \(y\in K_\nu\),
\[
h(y)-\delta\le g(y)\le h(y)+\delta.
\]
Hence, for every \(x\in K_\mu\),
\[
e^{-\delta/\varepsilon}
\int_{K_\nu} e^{(h(y)-c(x,y))/\varepsilon}\,d\nu(y)
\le
\int_{K_\nu} e^{(g(y)-c(x,y))/\varepsilon}\,d\nu(y)
\le
e^{\delta/\varepsilon}
\int_{K_\nu} e^{(h(y)-c(x,y))/\varepsilon}\,d\nu(y).
\]
Taking logarithms and multiplying by \(-\varepsilon\) gives
\[
|\mathcal T_\nu^\varepsilon g(x)-\mathcal T_\nu^\varepsilon h(x)|
\le
\delta
\qquad\text{for every }x\in K_\mu.
\]
Taking the supremum over \(x\in K_\mu\) proves the claim.
\end{proof}

The next proposition shows that the normalized target-side entropic potential
can be uniformly approximated by functions from the centered pullback class.

\begin{proposition}[Approximation of the normalized entropic potential]
\label{prop:potential-approximation}
Assume that \(\mathcal F\subset C(\mathbb R^n,\mathbb R)\) is dense under the
ucc topology and that the feature map \(\varphi:K_\nu\to\mathbb R^n\) satisfies
Assumption~\ref{ass-feature-regularity}. 
Let \(g_\varepsilon^\star\in C(K_\nu)\) be the normalized target-side
Schr\"odinger potential, chosen so that
\[
\int_{K_\nu} g_\varepsilon^\star\,d\nu=0,
\]
and set
\[
f_\varepsilon^\star:=\mathcal T_\nu^\varepsilon g_\varepsilon^\star.
\]
Then there exists a sequence \((h_m)_{m\in\mathbb N}\subset \mathsf C_\nu(\varphi^*\mathcal F)\) such
that, as \(m\to\infty\),
\[
\|h_m-g_\varepsilon^\star\|_{L^\infty(K_\nu)}\to 0,
\qquad
\|\mathcal T_\nu^\varepsilon h_m-f_\varepsilon^\star\|_{L^\infty(K_\mu)}\to 0,
\]
and
\[
\mathcal J_\varepsilon(h_m)\to\mathrm{OT}_\varepsilon.
\]
\end{proposition}

\begin{proof}
By Assumption~\ref{ass-feature-regularity} and the transfer principle from
Appendix~\ref{app:feature-map-review}, the pullback class
\(\varphi^*\mathcal F\) is dense in \(C(K_\nu)\) with respect to
\(\|\cdot\|_{L^\infty(K_\nu)}\). By
Proposition~\ref{prop:dual-attainment-review}, in the compact-support setting
we may choose the normalized target-side Schr\"odinger potential
\(g_\varepsilon^\star\) as an element of \(C(K_\nu)\). Since
\[
\int_{K_\nu} g_\varepsilon^\star\,d\nu=0,
\]
Lemma~\ref{lem:centered-pullback-class} yields a sequence
\((h_m)_{m\in\mathbb N}\subset\mathsf C_\nu(\varphi^*\mathcal F)\) such that
\[
\|h_m-g_\varepsilon^\star\|_{L^\infty(K_\nu)}\to 0.
\]

Applying Lemma~\ref{lem:supnorm-stability}, we obtain
\[
\|\mathcal T_\nu^\varepsilon h_m
-\mathcal T_\nu^\varepsilon g_\varepsilon^\star\|_{L^\infty(K_\mu)}
\le
\|h_m-g_\varepsilon^\star\|_{L^\infty(K_\nu)}
\to 0.
\]
Since
\[
\mathcal T_\nu^\varepsilon g_\varepsilon^\star=f_\varepsilon^\star,
\]
it follows that
\[
\|\mathcal T_\nu^\varepsilon h_m-f_\varepsilon^\star\|_{L^\infty(K_\mu)}
\to 0.
\]

It remains to show convergence of the semidual values. Since
\(h_m\in\mathsf C_\nu(\varphi^*\mathcal F)\), each \(h_m\) is centered:
\[
\int_{K_\nu}h_m\,d\nu=0.
\]
Together with the normalization of \(g_\varepsilon^\star\), this gives
\[
\int_{K_\nu}(h_m-g_\varepsilon^\star)\,d\nu=0.
\]
Therefore,
\begin{align*}
|\mathcal J_\varepsilon(h_m)-\mathcal J_\varepsilon(g_\varepsilon^\star)|
&=
\left|
\int_{K_\mu}
\bigl(
\mathcal T_\nu^\varepsilon h_m
-
\mathcal T_\nu^\varepsilon g_\varepsilon^\star
\bigr)\,d\mu
\right| \\
&\le
\|\mathcal T_\nu^\varepsilon h_m
-\mathcal T_\nu^\varepsilon g_\varepsilon^\star\|_{L^\infty(K_\mu)} \\
&\le
\|h_m-g_\varepsilon^\star\|_{L^\infty(K_\nu)}
\to 0.
\end{align*}
Since \(g_\varepsilon^\star\) is a semidual maximizer,
\[
\mathcal J_\varepsilon(g_\varepsilon^\star)=\mathrm{OT}_\varepsilon,
\]
we conclude that
\[
\mathcal J_\varepsilon(h_m)\to\mathrm{OT}_\varepsilon.
\]
\end{proof}

We next show how approximation of the semidual value translates into
convergence of the induced Gibbs plans.

\begin{proposition}[Gap identity and convergence of Gibbs plans]
\label{prop:dual-gap-kl}
For every bounded measurable function \(g:K_\nu\to\mathbb R\), define
\[
d\pi_g^\varepsilon(x,y)
:=
\exp\!\left(
\frac{\mathcal T_\nu^\varepsilon g(x)+g(y)-c(x,y)}{\varepsilon}
\right)\,d\mu(x)\,d\nu(y).
\]
Then \(\pi_g^\varepsilon\) is a probability measure with first marginal \(\mu\),
and
\begin{equation}
\label{eq:dual-gap-kl}
\varepsilon\,\mathrm{KL}(\pi_\varepsilon^\star \,\|\, \pi_g^\varepsilon)
=
\mathcal J_\varepsilon(g_\varepsilon^\star)-\mathcal J_\varepsilon(g).
\end{equation}
Consequently,
\begin{equation}
\label{eq:tv-bound}
\|\pi_g^\varepsilon-\pi_\varepsilon^\star\|_{\mathrm{TV}}
\le
\sqrt{
\frac{
\mathcal J_\varepsilon(g_\varepsilon^\star)-\mathcal J_\varepsilon(g)
}{2\varepsilon}
}.
\end{equation}
\end{proposition}

\begin{proof}
By definition of \(\mathcal T_\nu^\varepsilon g\), for every \(x\in K_\mu\),
\[
\int_{K_\nu}
\exp\!\left(
\frac{\mathcal T_\nu^\varepsilon g(x)+g(y)-c(x,y)}{\varepsilon}
\right)\,d\nu(y)=1.
\]
Integrating this identity with respect to \(\mu\) shows that
\(\pi_g^\varepsilon\) is a probability measure. The same identity also implies
that its first marginal is \(\mu\).

Next, using the Gibbs representation of \(\pi_\varepsilon^\star\), we have
\[
\log\frac{d\pi_\varepsilon^\star}{d\pi_g^\varepsilon}(x,y)
=
\frac{
f_\varepsilon^\star(x)+g_\varepsilon^\star(y)
-
\mathcal T_\nu^\varepsilon g(x)-g(y)
}{\varepsilon}.
\]
Integrating with respect to \(\pi_\varepsilon^\star\), and using that the
marginals of \(\pi_\varepsilon^\star\) are \(\mu\) and \(\nu\), gives
\begin{align*}
\varepsilon\,\mathrm{KL}(\pi_\varepsilon^\star \,\|\, \pi_g^\varepsilon)
&=
\int_{K_\mu}
\bigl(f_\varepsilon^\star-\mathcal T_\nu^\varepsilon g\bigr)\,d\mu
+
\int_{K_\nu}
\bigl(g_\varepsilon^\star-g\bigr)\,d\nu \\
&=
\mathcal J_\varepsilon(g_\varepsilon^\star)-\mathcal J_\varepsilon(g),
\end{align*}
because \(f_\varepsilon^\star=\mathcal T_\nu^\varepsilon g_\varepsilon^\star\).
Since \(g_\varepsilon^\star\) is a semidual maximizer,
\[
\mathcal J_\varepsilon(g_\varepsilon^\star)-\mathcal J_\varepsilon(g)\ge 0.
\]
The total-variation bound then follows from Pinsker's inequality, using the
probabilistic convention for total variation.
\end{proof}

We can now conclude the proof of Theorem~\ref{thm:main-fixed-eps-plan}.

\begin{proof}[Proof of Theorem~\ref{thm:main-fixed-eps-plan}]
By Proposition~\ref{prop:potential-approximation}, there exists a sequence
\((g_m)_{m\in\mathbb N}\subset\mathsf C_\nu(\varphi^*\mathcal F)\) such that
\[
\mathcal J_\varepsilon(g_m)\to\mathrm{OT}_\varepsilon
\qquad\text{as }m\to\infty.
\]
For each \(m\), let
\[
\pi_m^\varepsilon:=\pi_{g_m}^\varepsilon
\]
denote the Gibbs plan induced by \(g_m\). Applying
Proposition~\ref{prop:dual-gap-kl}, we obtain
\[
\varepsilon\,\mathrm{KL}(\pi_\varepsilon^\star \,\|\, \pi_m^\varepsilon)
=
\mathcal J_\varepsilon(g_\varepsilon^\star)-\mathcal J_\varepsilon(g_m).
\]
Since \(g_\varepsilon^\star\) is a semidual maximizer,
\[
\mathcal J_\varepsilon(g_\varepsilon^\star)=\mathrm{OT}_\varepsilon,
\]
it follows that
\[
\varepsilon\,\mathrm{KL}(\pi_\varepsilon^\star \,\|\, \pi_m^\varepsilon)
=
\mathrm{OT}_\varepsilon-\mathcal J_\varepsilon(g_m)
\to 0.
\]
Hence
\[
\mathrm{KL}(\pi_\varepsilon^\star \,\|\, \pi_m^\varepsilon)\to 0.
\]
Using \eqref{eq:tv-bound}, we also obtain
\[
\|\pi_m^\varepsilon-\pi_\varepsilon^\star\|_{\mathrm{TV}}
\le
\sqrt{
\frac{\mathrm{OT}_\varepsilon-\mathcal J_\varepsilon(g_m)}{2\varepsilon}
}
\to 0.
\]
Finally, total variation convergence implies weak convergence. Therefore,
\[
\pi_m^\varepsilon\rightharpoonup \pi_\varepsilon^\star.
\]
\end{proof}
\subsection{Proof of Proposition~\ref{thm:main-fixed-eps-map}}
\label{app:proof-main-fixed-eps-map}
We work under the assumptions of Proposition~\ref{thm:main-fixed-eps-map}.
Since \((\mathcal M,g)\) is Cartan--Hadamard, \((\mathcal M,d)\) is a global
NPC space. We use standard facts about barycenters in such spaces: probability
measures in \(\mathcal P_2(\mathcal M)\), and in particular compactly supported
probability measures, admit unique barycenters for the squared-distance
functional, the barycenter map is \(1\)-Lipschitz with respect to \(W_1\), and
the variance inequality holds
\cite[Propositions~4.3, 4.4 and Theorem~6.3]{Sturm2003}.

Set
\[
D_\nu
:=
\operatorname{diam}(K_\nu)
:=
\max_{y,y'\in K_\nu} d(y,y') < \infty .
\]
The maximum is finite because \(K_\nu\) is compact and \(d\) is continuous.

For any probability measure
\(\pi\in\mathcal P(K_\mu\times K_\nu)\) with first marginal \(\mu\), the
disintegration theorem on standard Borel spaces gives a regular conditional
kernel \(x\mapsto \pi_x\in\mathcal P(K_\nu)\), unique \(\mu\)-almost surely,
such that
\[
\pi(dx,dy)=\mu(dx)\,\pi_x(dy).
\]
Since \(\pi_x\) is supported on the compact set \(K_\nu\), we have
\(\pi_x\in\mathcal P_2(\mathcal M)\) for \(\mu\)-almost every \(x\). Hence its
barycenter is well defined and unique:
\[
T_\pi(x)
:=
\operatorname{bar}(\pi_x)
:=
\operatorname*{arg\,min}_{z\in\mathcal M}
\frac12\int_{\mathcal M} d(z,y)^2\,\pi_x(dy).
\]
We will write
\[
T_\varepsilon:=T_{\pi_\varepsilon^\star},
\qquad
\widehat T_m^\varepsilon:=T_{\pi_m^\varepsilon}
\]
for the barycentric projections associated with the optimal entropic coupling
and the approximating Gibbs plans, respectively.

Before proving convergence, we first verify that these barycentric projections
are measurable. This is needed because the estimates below involve integrals of
the form
\[
\int_{K_\mu} d\bigl(T_\pi(x),T_{\widetilde\pi}(x)\bigr)^q\,\mu(dx),
\qquad q\in\{1,2\},
\]
so the maps \(T_\pi\) must be well-defined as measurable maps, up to
\(\mu\)-null sets.

\begin{lemma}[Measurability of barycentric projections]
\label{lem:barycentric-measurability}
For every probability measure
\(\pi\in\mathcal P(K_\mu\times K_\nu)\) with first marginal \(\mu\), the map
\[
T_\pi(x)=\operatorname{bar}(\pi_x)
\]
is Borel measurable, up to modification on a \(\mu\)-null set.
\end{lemma}

\begin{proof}
Since \(K_\mu\) and \(K_\nu\) are compact metric spaces, they are standard
Borel spaces. Therefore, by the existence of regular conditional distributions
and the disintegration theorem \cite[Theorems~6.3--6.4]{Kallenberg2002}, there
exists a \(\mu\)-almost surely unique Markov kernel
\(x\mapsto \pi_x\in\mathcal P(K_\nu)\) such that
\[
\pi(dx,dy)=\mu(dx)\,\pi_x(dy).
\]

We first view this kernel as a measurable map into
\(\mathcal P(K_\nu)\) equipped with the weak Borel \(\sigma\)-algebra. By
definition of a Markov kernel, for every Borel set \(B\subset K_\nu\), the map
\[
x\mapsto \pi_x(B)
\]
is measurable. Since \(K_\nu\) is compact metric, it is separable metrizable,
and the Borel \(\sigma\)-algebra on \(\mathcal P(K_\nu)\) induced by the weak
topology is generated by the evaluation maps
\[
e_B:\rho\mapsto \rho(B),
\qquad B\in\mathcal B(K_\nu);
\]
see \cite[Proposition~7.25]{BertsekasShreve1978}. For each such \(B\), the
composition of \(x\mapsto\pi_x\) with \(e_B\) is precisely
\[
(e_B\circ (x\mapsto\pi_x))(x)=\pi_x(B),
\]
which is measurable by the Markov-kernel property. Since the maps \(e_B\)
generate the weak Borel \(\sigma\)-algebra on \(\mathcal P(K_\nu)\), it follows
that \(x\mapsto \pi_x\) is Borel measurable as a map into
\(\mathcal P(K_\nu)\) equipped with its weak Borel \(\sigma\)-algebra.

Since \(K_\nu\) is compact, \(W_1\)-convergence on \(\mathcal P(K_\nu)\) is
equivalent to weak convergence: by the standard characterization of
Wasserstein convergence, \(W_1\)-convergence is equivalent to weak convergence
plus convergence of first moments, and the first-moment condition is automatic
on compact sets \cite[Theorem~6.9]{Villani2009}. Hence the weak and \(W_1\)
topologies induce the same Borel \(\sigma\)-algebra on \(\mathcal P(K_\nu)\),
so \(x\mapsto \pi_x\) is also measurable as a map into
\((\mathcal P(K_\nu),W_1)\).

The barycenter map is \(1\)-Lipschitz with respect to \(W_1\) in global NPC
spaces \cite[Theorem~6.3]{Sturm2003}. Therefore the composition
\[
x\mapsto \pi_x\mapsto \operatorname{bar}(\pi_x)
\]
is Borel measurable. Thus \(T_\pi\) is Borel measurable, up to modification on
a \(\mu\)-null set.
\end{proof}

We use the probabilistic convention for total variation: for probability
measures \(\rho,\widetilde\rho\) on a common measurable space,
\[
\|\rho-\widetilde\rho\|_{\mathrm{TV}}
:=
\sup_A |\rho(A)-\widetilde\rho(A)|.
\]
Equivalently, if \(\lambda\) is any finite measure such that
\(\rho,\widetilde\rho\ll\lambda\), then
\[
\|\rho-\widetilde\rho\|_{\mathrm{TV}}
=
\frac12
\int
\left|
\frac{d\rho}{d\lambda}
-
\frac{d\widetilde\rho}{d\lambda}
\right|d\lambda .
\]
We will also use the variational representation
\[
\|\rho-\widetilde\rho\|_{\mathrm{TV}}
=
\sup_{\substack{g\ \mathrm{measurable}\\ |g|\le 1/2}}
\left\{
\int g\,d\rho-\int g\,d\widetilde\rho
\right\}
=
\frac12
\sup_{\substack{f\ \mathrm{measurable}\\ \|f\|_\infty\le 1}}
\left|
\int f\,d(\rho-\widetilde\rho)
\right|.
\]
This is the total-variation special case of the variational representation of
\(f\)-divergences, obtained by taking \(f(t)=\frac12|t-1|\); see
\cite[Theorem~7.26 and Example~7.3]{PolyanskiyWu2025}.

We next record a fiberwise identity for total variation. This is a direct
consequence of disintegration and the Radon--Nikodym theorem.

\begin{lemma}[Fiberwise total variation identity]
\label{lem:fiber-tv-identity}
Let \(\pi,\widetilde\pi\in\mathcal P(K_\mu\times K_\nu)\) have first marginal
\(\mu\), and write
\[
\pi(dx,dy)=\mu(dx)\,\pi_x(dy),
\qquad
\widetilde\pi(dx,dy)=\mu(dx)\,\widetilde\pi_x(dy).
\]
Then
\[
\|\pi-\widetilde\pi\|_{\mathrm{TV}}
=
\int_{K_\mu}\|\pi_x-\widetilde\pi_x\|_{\mathrm{TV}}\,\mu(dx).
\]
\end{lemma}
\begin{proof}
Let
\[
\lambda(dx,dy):=\mu(dx)\bigl(\pi_x+\widetilde\pi_x\bigr)(dy).
\]
Then \(\lambda\) is a finite measure and \(\pi,\widetilde\pi\ll\lambda\).
Indeed, if \(\lambda(A)=0\) for a Borel set
\(A\subset K_\mu\times K_\nu\), then, writing
\[
A_x:=\{y\in K_\nu:(x,y)\in A\},
\]
we have
\[
0=\lambda(A)
=
\int_{K_\mu}\bigl(\pi_x(A_x)+\widetilde\pi_x(A_x)\bigr)\,\mu(dx).
\]
Since the integrand is nonnegative, this implies
\(\pi_x(A_x)=\widetilde\pi_x(A_x)=0\) for \(\mu\)-almost every \(x\). Hence
\[
\pi(A)=\int_{K_\mu}\pi_x(A_x)\,\mu(dx)=0,
\qquad
\widetilde\pi(A)=\int_{K_\mu}\widetilde\pi_x(A_x)\,\mu(dx)=0.
\]
Thus \(\pi,\widetilde\pi\ll\lambda\).

By the Radon--Nikodym theorem for kernels on standard Borel spaces, we may
choose jointly measurable versions of the fiberwise densities
\[
r(x,y):=\frac{d\pi_x}{d(\pi_x+\widetilde\pi_x)}(y),
\qquad
\widetilde r(x,y):=
\frac{d\widetilde\pi_x}{d(\pi_x+\widetilde\pi_x)}(y).
\]
We claim that
\[
\frac{d\pi}{d\lambda}(x,y)=r(x,y),
\qquad
\frac{d\widetilde\pi}{d\lambda}(x,y)=\widetilde r(x,y).
\]
To verify the first identity, let \(\Phi\) be a bounded measurable function on
\(K_\mu\times K_\nu\). Then
\begin{align*}
\int_{K_\mu\times K_\nu}\Phi(x,y)r(x,y)\,\lambda(dx,dy)
&=
\int_{K_\mu}
\left[
\int_{K_\nu}
\Phi(x,y)r(x,y)
\,(\pi_x+\widetilde\pi_x)(dy)
\right]\mu(dx) \\
&=
\int_{K_\mu}
\left[
\int_{K_\nu}
\Phi(x,y)\,\pi_x(dy)
\right]\mu(dx) \\
&=
\int_{K_\mu\times K_\nu}\Phi(x,y)\,\pi(dx,dy).
\end{align*}
Therefore \(r=d\pi/d\lambda\). The proof that
\(\widetilde r=d\widetilde\pi/d\lambda\) is identical.

Using the density representation of total variation with respect to the common
dominating measure \(\lambda\), we obtain
\begin{align*}
2\|\pi-\widetilde\pi\|_{\mathrm{TV}}
&=
\int_{K_\mu\times K_\nu}
|r(x,y)-\widetilde r(x,y)|\,\lambda(dx,dy) \\
&=
\int_{K_\mu}
\left(
\int_{K_\nu}
|r(x,y)-\widetilde r(x,y)|
\,(\pi_x+\widetilde\pi_x)(dy)
\right)\mu(dx).
\end{align*}
For each fixed \(x\), the measure \(\pi_x+\widetilde\pi_x\) dominates both
\(\pi_x\) and \(\widetilde\pi_x\), with densities \(r(x,\cdot)\) and
\(\widetilde r(x,\cdot)\), respectively. Therefore, again by the density
representation of total variation,
\[
\int_{K_\nu}
|r(x,y)-\widetilde r(x,y)|
\,(\pi_x+\widetilde\pi_x)(dy)
=
2\|\pi_x-\widetilde\pi_x\|_{\mathrm{TV}}
\qquad
\text{for \(\mu\)-a.e. }x.
\]
Substituting this into the previous display gives
\[
2\|\pi-\widetilde\pi\|_{\mathrm{TV}}
=
2\int_{K_\mu}\|\pi_x-\widetilde\pi_x\|_{\mathrm{TV}}\,\mu(dx).
\]
Dividing by \(2\) proves the identity.
\end{proof}
The next estimate converts total variation convergence of plans with first
marginal \(\mu\) into \(L^1\) and \(L^2\) convergence of their barycentric
projections.

\begin{proposition}[From plan convergence to barycentric map convergence]
\label{prop:plan-to-map}
Let \(\pi,\widetilde\pi\in\mathcal P(K_\mu\times K_\nu)\) have first marginal
\(\mu\). Then
\begin{equation}
\label{eq:appendix-L1-map-bound}
\int_{K_\mu} d\bigl(T_\pi(x),T_{\widetilde\pi}(x)\bigr)\,\mu(dx)
\le
D_\nu\,\|\pi-\widetilde\pi\|_{\mathrm{TV}},
\end{equation}
and
\begin{equation}
\label{eq:appendix-L2-map-bound}
\int_{K_\mu} d\bigl(T_\pi(x),T_{\widetilde\pi}(x)\bigr)^2\,\mu(dx)
\le
2D_\nu^2\,\|\pi-\widetilde\pi\|_{\mathrm{TV}}.
\end{equation}
\end{proposition}

\begin{proof}
Write
\[
\pi(dx,dy)=\mu(dx)\,\pi_x(dy),
\qquad
\widetilde\pi(dx,dy)=\mu(dx)\,\widetilde\pi_x(dy).
\]
By Lemma~\ref{lem:barycentric-measurability}, the maps \(T_\pi\) and
\(T_{\widetilde\pi}\) are Borel measurable. Since \(d\) is continuous, the
functions
\[
x\mapsto d\bigl(T_\pi(x),T_{\widetilde\pi}(x)\bigr),
\qquad
x\mapsto d\bigl(T_\pi(x),T_{\widetilde\pi}(x)\bigr)^2
\]
are measurable.

By the \(W_1\)-contraction property of barycenters in global NPC spaces
\cite[Theorem~6.3]{Sturm2003},
\[
d\bigl(T_\pi(x),T_{\widetilde\pi}(x)\bigr)
\le
W_1(\pi_x,\widetilde\pi_x)
\qquad \mu\text{-a.e. }x.
\]
We now bound the fiberwise \(W_1\) distance by fiberwise total variation. Since both conditional
laws are supported on \(K_\nu\), they belong to \(\mathcal P_1(K_\nu)\). By the
Kantorovich--Rubinstein duality formula \cite[Theorem~5.10]{Villani2009},
\[
W_1(\pi_x,\widetilde\pi_x)
=
\sup_{\operatorname{Lip}(\phi)\le 1}
\left|
\int_{K_\nu}\phi\,d(\pi_x-\widetilde\pi_x)
\right|.
\]
Fix any \(1\)-Lipschitz \(\phi:K_\nu\to\mathbb R\). Since
\(\pi_x\) and \(\widetilde\pi_x\) are probability measures,
\[
(\pi_x-\widetilde\pi_x)(K_\nu)=0.
\]
Thus subtracting constants from \(\phi\) does not change the integral:
\[
\int_{K_\nu}\phi\,d(\pi_x-\widetilde\pi_x)
=
\int_{K_\nu}(\phi-a)\,d(\pi_x-\widetilde\pi_x)
\qquad \text{for every } a\in\mathbb R.
\]
Choose \(a:=\inf_{K_\nu}\phi\). Then
\[
0\le \phi-a\le \operatorname{osc}_{K_\nu}(\phi).
\]
If \(\operatorname{osc}_{K_\nu}(\phi)=0\), the integral is zero. Otherwise, set
\[
h:=\frac{\phi-a}{\operatorname{osc}_{K_\nu}(\phi)}.
\]
Then \(0\le h\le 1\). Since
\[
(\pi_x-\widetilde\pi_x)(K_\nu)=0,
\]
we may subtract \(1/2\) from \(h\) without changing the integral:
\[
\int_{K_\nu}h\,d(\pi_x-\widetilde\pi_x)
=
\int_{K_\nu}\left(h-\frac12\right)\,d(\pi_x-\widetilde\pi_x).
\]
Moreover, \(\left|h-\frac12\right|\le 1/2\). Hence, by the variational
representation of total variation stated above,
\[
\left|
\int_{K_\nu}h\,d(\pi_x-\widetilde\pi_x)
\right|
\le
\|\pi_x-\widetilde\pi_x\|_{\mathrm{TV}}.
\]
Therefore,
\[
\left|
\int_{K_\nu}\phi\,d(\pi_x-\widetilde\pi_x)
\right|
\le
\operatorname{osc}_{K_\nu}(\phi)\,
\|\pi_x-\widetilde\pi_x\|_{\mathrm{TV}}.
\]
Because \(\phi\) is \(1\)-Lipschitz and \(K_\nu\) has diameter \(D_\nu\),
\[
\operatorname{osc}_{K_\nu}(\phi)
:=
\sup_{y,y'\in K_\nu}|\phi(y)-\phi(y')|
\le
D_\nu.
\]
Thus
\[
\left|
\int_{K_\nu}\phi\,d(\pi_x-\widetilde\pi_x)
\right|
\le
D_\nu\,\|\pi_x-\widetilde\pi_x\|_{\mathrm{TV}}.
\]
Taking the supremum over all \(1\)-Lipschitz \(\phi\) yields
\[
W_1(\pi_x,\widetilde\pi_x)
\le
D_\nu\,\|\pi_x-\widetilde\pi_x\|_{\mathrm{TV}}.
\]
Combining this with the barycenter contraction gives
\[
d\bigl(T_\pi(x),T_{\widetilde\pi}(x)\bigr)
\le
D_\nu\,\|\pi_x-\widetilde\pi_x\|_{\mathrm{TV}}
\qquad \mu\text{-a.e. }x.
\]
Integrating over \(K_\mu\) and using Lemma~\ref{lem:fiber-tv-identity}, we get
\[
\int_{K_\mu} d\bigl(T_\pi(x),T_{\widetilde\pi}(x)\bigr)\,\mu(dx)
\le
D_\nu\int_{K_\mu}
\|\pi_x-\widetilde\pi_x\|_{\mathrm{TV}}\,\mu(dx)
=
D_\nu\,\|\pi-\widetilde\pi\|_{\mathrm{TV}}.
\]
This proves \eqref{eq:appendix-L1-map-bound}.

It remains to prove the \(L^2\) estimate. Fix \(y_0\in K_\nu\). By the
variance inequality for barycenters in global NPC spaces
\cite[Proposition~4.4]{Sturm2003}, for \(\mu\)-almost every \(x\),
\[
d\bigl(T_\pi(x),y_0\bigr)^2
\le
\int_{K_\nu} d(y_0,y)^2\,\pi_x(dy).
\]
Since \(\pi_x\) is supported on \(K_\nu\) and \(y_0\in K_\nu\), we have
\(d(y_0,y)\le D_\nu\) for all \(y\in K_\nu\). Therefore
\[
d\bigl(T_\pi(x),y_0\bigr)^2
\le
\int_{K_\nu}D_\nu^2\,\pi_x(dy)
=
D_\nu^2.
\]
Thus
\[
d\bigl(T_\pi(x),y_0\bigr)\le D_\nu.
\]
The same argument applied to \(\widetilde\pi_x\) gives
\[
d\bigl(T_{\widetilde\pi}(x),y_0\bigr)\le D_\nu.
\]
By the triangle inequality,
\[
d\bigl(T_\pi(x),T_{\widetilde\pi}(x)\bigr)
\le
d\bigl(T_\pi(x),y_0\bigr)
+
d\bigl(y_0,T_{\widetilde\pi}(x)\bigr)
\le
2D_\nu
\qquad \mu\text{-a.e. }x.
\]
Hence, setting
\[
a(x):=d\bigl(T_\pi(x),T_{\widetilde\pi}(x)\bigr),
\]
we have \(0\le a(x)\le 2D_\nu\), and therefore
\[
a(x)^2\le 2D_\nu a(x).
\]
Equivalently,
\[
d\bigl(T_\pi(x),T_{\widetilde\pi}(x)\bigr)^2
\le
2D_\nu\,d\bigl(T_\pi(x),T_{\widetilde\pi}(x)\bigr).
\]
Integrating and applying the \(L^1\) bound \eqref{eq:appendix-L1-map-bound},
\[
\int_{K_\mu} d\bigl(T_\pi(x),T_{\widetilde\pi}(x)\bigr)^2\,\mu(dx)
\le
2D_\nu
\int_{K_\mu} d\bigl(T_\pi(x),T_{\widetilde\pi}(x)\bigr)\,\mu(dx)
\le
2D_\nu^2\,\|\pi-\widetilde\pi\|_{\mathrm{TV}}.
\]
This proves \eqref{eq:appendix-L2-map-bound}.
\end{proof}

We now prove the proposition.

\begin{proof}[Proof of Proposition~\ref{thm:main-fixed-eps-map}]
By Theorem~\ref{thm:main-fixed-eps-plan}, there exists a sequence
\((g_m)_{m\in\mathbb N}\subset\mathsf C_\nu(\varphi^*\mathcal F)\) such that,
if \(\pi_m^\varepsilon\) denotes the induced Gibbs plan, then
\[
\|\pi_m^\varepsilon-\pi_\varepsilon^\star\|_{\mathrm{TV}}
\xrightarrow[m\to\infty]{}0.
\]
By Lemma~\ref{lem:barycentric-measurability}, both
\[
\widehat T_m^\varepsilon=T_{\pi_m^\varepsilon},
\qquad
T_\varepsilon=T_{\pi_\varepsilon^\star}
\]
are Borel measurable, so the \(L^1(\mu)\) and \(L^2(\mu)\) distances below are
well defined.

Applying Proposition~\ref{prop:plan-to-map} with
\[
\pi=\pi_m^\varepsilon,
\qquad
\widetilde\pi=\pi_\varepsilon^\star,
\]
we obtain
\[
\int_{K_\mu}
d\bigl(\widehat T_m^\varepsilon(x),T_\varepsilon(x)\bigr)\,\mu(dx)
\le
D_\nu\,\|\pi_m^\varepsilon-\pi_\varepsilon^\star\|_{\mathrm{TV}}
\xrightarrow[m\to\infty]{}0,
\]
and
\[
\int_{K_\mu}
d\bigl(\widehat T_m^\varepsilon(x),T_\varepsilon(x)\bigr)^2\,\mu(dx)
\le
2D_\nu^2\,
\|\pi_m^\varepsilon-\pi_\varepsilon^\star\|_{\mathrm{TV}}
\xrightarrow[m\to\infty]{}0.
\]
Therefore
\[
\widehat T_m^\varepsilon\to T_\varepsilon
\qquad\text{in }L^2(\mu).
\]
The \(L^1(\mu)\) convergence follows either from the first bound above or from
Cauchy--Schwarz, since \(\mu\) is a probability measure.
\end{proof}

\subsection{Proof of Proposition~\ref{prop:main-fixed-eps-heat}}
\label{app:proof-main-fixed-eps-heat}

\begin{proof}
Let \((g_m)_{m\in\mathbb N}\subset \mathsf C_\nu(\varphi^*\mathcal F)\) be the
approximating sequence given by Theorem~\ref{thm:main-fixed-eps-plan}. Thus the
associated Gibbs plans satisfy
\[
\|\pi_m^\varepsilon-\pi_\varepsilon^\star\|_{\mathrm{TV}}
\xrightarrow[m\to\infty]{}0.
\]

Fix \(t>0\). By the heat-kernel facts recalled in
Appendix~\ref{app:heat-smoothed-stochastic-completeness}, the heat kernel
\(p_t(y,z)\) associated with the Riemannian heat semigroup is nonnegative and,
under stochastic completeness, satisfies
\[
\int_{\mathcal M}p_t(y,z)\,\mathrm{vol}_{\mathcal M}(dz)=1.
\]
Hence
\[
P_t(y,dz):=p_t(y,z)\,\mathrm{vol}_{\mathcal M}(dz)
\]
is a conservative Markov kernel on \(\mathcal M\). This is the only point in
the proof where stochastic completeness is used.

For a probability measure
\(\eta\in\mathcal P(\mathcal M\times\mathcal M)\), let \(\eta P_t\) denote the
measure obtained by applying \(P_t\) in the second variable:
\[
(\eta P_t)(C)
:=
\int_{\mathcal M\times\mathcal M}
\int_{\mathcal M}\mathbbm 1_C(x,z)\,P_t(y,dz)\,\eta(dx,dy),
\qquad
C\in\mathcal B(\mathcal M\times\mathcal M).
\]
Equivalently, if \((x,y)\sim\eta\) and, conditionally on \(y\),
\(z\sim P_t(y,\cdot)\), then \((x,z)\sim\eta P_t\). By the definition of the
heat-smoothed transport surrogate,
\[
\Pi_{m,t}^\varepsilon=\pi_m^\varepsilon P_t,
\qquad
\Pi_{\varepsilon,t}^\star=\pi_\varepsilon^\star P_t.
\]

We first prove absolute continuity of the smoothed conditional laws. Since
\(\mathcal M\) is a Riemannian manifold, it is a standard Borel space, and the
plans admit regular conditional distributions. Let
\(\pi_{m,x}^\varepsilon\) be a disintegration of \(\pi_m^\varepsilon\) with
respect to its first marginal \(\mu\). Then, for \(\mu\)-a.e. \(x\),
\[
Q_{m,t}^\varepsilon(x,dz)
=
\int_{\mathcal M}P_t(y,dz)\,\pi_{m,x}^\varepsilon(dy).
\]
Using the heat-kernel representation recalled in
Appendix~\ref{app:heat-smoothed-stochastic-completeness}, for every measurable
\(B\subseteq\mathcal M\) we have
\[
Q_{m,t}^\varepsilon(x,B)
=
\int_{\mathcal M}
\left(
\int_B p_t(y,z)\,\mathrm{vol}_{\mathcal M}(dz)
\right)
\pi_{m,x}^\varepsilon(dy).
\]
Since \(p_t(y,z)\ge0\), Tonelli's theorem gives
\[
Q_{m,t}^\varepsilon(x,B)
=
\int_B
\left(
\int_{\mathcal M}p_t(y,z)\,\pi_{m,x}^\varepsilon(dy)
\right)
\mathrm{vol}_{\mathcal M}(dz).
\]
Hence
\[
Q_{m,t}^\varepsilon(x,dz)
=
\left(
\int_{\mathcal M}p_t(y,z)\,\pi_{m,x}^\varepsilon(dy)
\right)
\mathrm{vol}_{\mathcal M}(dz).
\]
Therefore
\[
Q_{m,t}^\varepsilon(x,\cdot)\ll \mathrm{vol}_{\mathcal M}
\qquad
\text{for }\mu\text{-a.e. }x,
\]
with density
\[
q_{m,t}^\varepsilon(x,z)
=
\int_{\mathcal M}p_t(y,z)\,\pi_{m,x}^\varepsilon(dy).
\]
The same argument applied to a disintegration
\(\pi_{\varepsilon,x}^\star\) of \(\pi_\varepsilon^\star\) shows that
\[
Q_{\varepsilon,t}^\star(x,\cdot)\ll \mathrm{vol}_{\mathcal M}
\qquad
\text{for }\mu\text{-a.e. }x.
\]

We now prove the total-variation contraction estimate. Let
\(C\in\mathcal B(\mathcal M\times\mathcal M)\), and define
\[
H_C(x,y)
:=
\int_{\mathcal M}\mathbbm 1_C(x,z)\,P_t(y,dz).
\]
Since \(P_t\) is a Markov kernel, \(0\le H_C\le 1\). Moreover,
\[
\Pi_{m,t}^\varepsilon(C)-\Pi_{\varepsilon,t}^\star(C)
=
\int_{\mathcal M\times\mathcal M}
H_C(x,y)\,
d(\pi_m^\varepsilon-\pi_\varepsilon^\star)(x,y).
\]
We now prove the total-variation contraction estimate. Let
\(C\in\mathcal B(\mathcal M\times\mathcal M)\), and define
\[
H_C(x,y)
:=
\int_{\mathcal M}\mathbbm 1_C(x,z)\,P_t(y,dz).
\]
Since \(P_t\) is a Markov kernel, \(0\le H_C\le 1\). Moreover,
\[
\Pi_{m,t}^\varepsilon(C)-\Pi_{\varepsilon,t}^\star(C)
=
\int_{\mathcal M\times\mathcal M}
H_C(x,y)\,
d(\pi_m^\varepsilon-\pi_\varepsilon^\star)(x,y).
\]
Since \(\pi_m^\varepsilon\) and \(\pi_\varepsilon^\star\) are probability
measures,
\[
(\pi_m^\varepsilon-\pi_\varepsilon^\star)(\mathcal M\times\mathcal M)=0.
\]
Thus
\[
\int H_C\,d(\pi_m^\varepsilon-\pi_\varepsilon^\star)
=
\int \left(H_C-\frac12\right)
\,d(\pi_m^\varepsilon-\pi_\varepsilon^\star),
\]
and \(\left|H_C-\frac12\right|\le 1/2\). By the variational representation of
total variation,
\[
\left|
\Pi_{m,t}^\varepsilon(C)-\Pi_{\varepsilon,t}^\star(C)
\right|
\le
\|\pi_m^\varepsilon-\pi_\varepsilon^\star\|_{\mathrm{TV}}.
\]
Taking the supremum over
\(C\in\mathcal B(\mathcal M\times\mathcal M)\) gives
\[
\|\Pi_{m,t}^\varepsilon-\Pi_{\varepsilon,t}^\star\|_{\mathrm{TV}}
\le
\|\pi_m^\varepsilon-\pi_\varepsilon^\star\|_{\mathrm{TV}}.
\]
Taking the supremum over
\(C\in\mathcal B(\mathcal M\times\mathcal M)\) gives
\[
\|\Pi_{m,t}^\varepsilon-\Pi_{\varepsilon,t}^\star\|_{\mathrm{TV}}
\le
\|\pi_m^\varepsilon-\pi_\varepsilon^\star\|_{\mathrm{TV}}.
\]
Together with Theorem~\ref{thm:main-fixed-eps-plan}, this yields
\[
\|\Pi_{m,t}^\varepsilon-\Pi_{\varepsilon,t}^\star\|_{\mathrm{TV}}
\xrightarrow[m\to\infty]{}0.
\]
In particular,
\[
\Pi_{m,t}^\varepsilon\rightharpoonup\Pi_{\varepsilon,t}^\star
\qquad
(m\to\infty)
\]
for every fixed \(t>0\).

It remains to show that the population-level heat smoothing has vanishing bias:
\[
\Pi_{\varepsilon,t}^\star\rightharpoonup\pi_\varepsilon^\star
\qquad
(t\downarrow0).
\]
Let \(\Phi\in C_b(\mathcal M\times\mathcal M)\). Then
\[
\int_{\mathcal M\times\mathcal M}\Phi(x,z)\,
\Pi_{\varepsilon,t}^\star(dx,dz)
=
\int_{\mathcal M\times\mathcal M}
(P_t\Phi_x)(y)\,\pi_\varepsilon^\star(dx,dy),
\qquad
\Phi_x(z):=\Phi(x,z).
\]
By the extension of the small-time heat-kernel limit to bounded continuous
functions recalled in
Appendix~\ref{app:heat-smoothed-stochastic-completeness}, for every fixed
\((x,y)\in\mathcal M\times\mathcal M\),
\[
(P_t\Phi_x)(y)
\xrightarrow[t\downarrow0]{}
\Phi_x(y)=\Phi(x,y),
\]
because \(\Phi_x\in C_b(\mathcal M)\). Moreover, since \(P_t\) is a Markov
kernel,
\[
|(P_t\Phi_x)(y)-\Phi(x,y)|
\le
2\|\Phi\|_\infty.
\]
Dominated convergence therefore gives
\[
\int \Phi\,d\Pi_{\varepsilon,t}^\star
\xrightarrow[t\downarrow0]{}
\int \Phi\,d\pi_\varepsilon^\star.
\]
Hence
\[
\Pi_{\varepsilon,t}^\star\rightharpoonup\pi_\varepsilon^\star
\qquad
(t\downarrow0).
\]

Finally, let \(t_m\downarrow0\). For any
\(\Phi\in C_b(\mathcal M\times\mathcal M)\),
\[
\left|
\int\Phi\,d\Pi_{m,t_m}^\varepsilon
-
\int\Phi\,d\pi_\varepsilon^\star
\right|
\le
\left|
\int\Phi\,d(\Pi_{m,t_m}^\varepsilon-\Pi_{\varepsilon,t_m}^\star)
\right|
+
\left|
\int\Phi\,d\Pi_{\varepsilon,t_m}^\star
-
\int\Phi\,d\pi_\varepsilon^\star
\right|.
\]
For the first term, using the probabilistic convention for total variation,
\[
\left|
\int\Phi\,d(\Pi_{m,t_m}^\varepsilon-\Pi_{\varepsilon,t_m}^\star)
\right|
\le
2\|\Phi\|_\infty
\|\Pi_{m,t_m}^\varepsilon-\Pi_{\varepsilon,t_m}^\star\|_{\mathrm{TV}}.
\]
By the contraction estimate proved above,
\[
\|\Pi_{m,t_m}^\varepsilon-\Pi_{\varepsilon,t_m}^\star\|_{\mathrm{TV}}
\le
\|\pi_m^\varepsilon-\pi_\varepsilon^\star\|_{\mathrm{TV}}
\xrightarrow[m\to\infty]{}0.
\]
The second term tends to zero because \(t_m\downarrow0\) and
\[
\Pi_{\varepsilon,t}^\star\rightharpoonup\pi_\varepsilon^\star
\qquad
(t\downarrow0).
\]
Therefore
\[
\Pi_{m,t_m}^\varepsilon\rightharpoonup\pi_\varepsilon^\star,
\]
which completes the proof.
\end{proof}